\documentclass[11pt]{article}
\usepackage[margin=1in, paper=letterpaper]{geometry}
\usepackage[utf8]{inputenc}

\usepackage{amsmath}
\usepackage{amssymb}
\usepackage{amsthm} \usepackage{mathtools}

\usepackage[numbers]{natbib}
\PassOptionsToPackage{hyphens}{url}\usepackage[backref]{hyperref}
\renewcommand*{\backref}[1]{}
\renewcommand*{\backrefalt}[4]{{\footnotesize[\ifcase #1 Not cited.\else Cited on #2.\fi ]}}

\usepackage{dsfont}
\usepackage{todonotes}
\usepackage{authblk}
\usepackage[capitalise]{cleveref}
\newtheorem{theorem}{Theorem}
\newtheorem{lemma}[theorem]{Lemma}
\newtheorem{corollary}[theorem]{Corollary}
\newtheorem{claim}[theorem]{Claim}
\newtheorem{remark}[theorem]{Remark}
\newtheorem{definition}[theorem]{Definition}
\crefname{claim}{Claim}{Claims}

\newcommand{\pr}[1]{\mathrm{Pr}\left[#1\right]}
\newcommand{\var}[1]{\mathrm{Var}\left[#1\right]}
\newcommand{\cov}[1]{\mathrm{Cov}\left[#1\right]}
\newcommand{\abs}[1]{\left\lvert #1 \right\rvert}

\newcommand{\twoNorm}[1]{\left\lVert #1 \right\rVert_2}

\newcommand{\expect}[1]{\mathbb{E}\left[#1\right]}
\newcommand{\given}{\,\middle|\,}
\newcommand{\diff}{\mathrm{d}}

\newcommand{\Normal}[1]{\mathcal{N}\left(#1\right)}

\newcommand{\CC}{\mathcal{C}}

\usepackage{amsmath,amsfonts,bm}

\def\1{\bm{1}}

\def\ra{{\textnormal{A}}}
\def\rb{{\textnormal{B}}}
\def\rc{{\textnormal{C}}}

\def\rx{{\textnormal{X}}}
\def\ry{{\textnormal{Y}}}

\def\rvx{{\mathbf{X}}}
\def\rvy{{\mathbf{Y}}}

\def\vzero{{\bm{0}}}

\def\vp{{\bm{p}}}

\def\vv{{\bm{v}}}

\def\vx{{\bm{x}}}
\def\vy{{\bm{y}}}
\def\vz{{\bm{z}}}

\def\mI{{\bm{I}}}

\def\mX{{\bm{X}}}

\DeclareMathAlphabet{\mathsfit}{\encodingdefault}{\sfdefault}{m}{sl}
\SetMathAlphabet{\mathsfit}{bold}{\encodingdefault}{\sfdefault}{bx}{n}
\newcommand{\tens}[1]{\bm{\mathsfit{#1}}}

\def\tW{{\tens{W}}}
\def\tX{{\tens{X}}}

\def\gH{{\mathcal{H}}}

\def\gO{{\mathcal{O}}}

\newcommand{\N}{\mathbb{N}}
\newcommand{\R}{\mathbb{R}}
\newcommand{\bigO}{\mathcal{O}}
\newcommand{\NormalDist}{\mathcal{N}}
\newcommand{\dBall}{\mathcal{B}_{\infty}^d}
 
\title{On the Multidimensional Random Subset Sum Problem}
\author[1]{Luca Becchetti}
\author[2]{Arthur C. W. da Cunha}
\author[3]{Andrea Clementi}
\author[2]{Francesco d'Amore}
\author[2]{Hicham Lesfari}
\author[2]{Emanuele Natale}
\author[4]{Luca Trevisan}
\affil[1]{University of Rome La Sapienza}
\affil[2]{INRIA, CNRS, Université Côte d'Azur}
\affil[3]{University of Rome Tor Vergata}
\affil[4]{Bocconi University}
\date{}

\begin{document}

\maketitle

\begin{abstract}
In the Random Subset Sum Problem, given $n$ i.i.d. random variables $X_1, ..., X_n$, we wish to approximate any point $z \in [-1,1]$ as the sum of a suitable subset $X_{i_1(z)}, ..., X_{i_s(z)}$ of them, up to error $\varepsilon$.
Despite its simple statement, this problem is of fundamental interest to both theoretical computer science and statistical mechanics. More recently, it gained renewed attention for its implications in the theory of Artificial Neural Networks.
An obvious multidimensional generalisation of the problem is to consider $n$ i.i.d.\ $d$-dimensional random vectors, with the objective of approximating every point $\mathbf{z} \in [-1,1]^d$.
In 1998, G.\ S.\ Lueker showed that, in the one-dimensional setting, $n=\mathcal{O}(\log \frac 1\varepsilon)$ samples guarantee the approximation property with high probability.

In this work, we prove that, in $d$ dimensions, $n = \mathcal{O}(d^3\log \frac 1\varepsilon \cdot (\log \frac 1\varepsilon + \log d))$ samples suffice for the approximation property to hold with high probability. 
As an application highlighting the potential interest of this result, we prove that a recently proposed neural network model
exhibits \emph{universality}: with high probability, the model can approximate any neural network within a polynomial overhead in the number of parameters.
\end{abstract}

\section{Introduction}
\label{sec:Intro}

In the \emph{Random Subset Sum Problem (RSSP)}, given a target value $z$, an error parameter $\varepsilon \in \R_{>0}$ and $n$ independent random variables $\rx_1, \rx_2, \ldots, \rx_n$, one is interested in estimating the probability that there exists a subset $S \subseteq [n]$ for which
$$
    \bigl\lvert z - \sum_{i \in S} \rx_i \bigr\rvert \le \varepsilon.
$$

Historically, the analysis of this problem was mainly motivated by research on the average case of its deterministic counterpart, the classic Subset Sum Problem, and the equivalent Number Partition Problem.
These investigations lead to a number of insightful results, mostly in the 80s and 90s \cite{lueker1982, karmarkar1986, lueker-exponentially-1998}. In addition, research on the phase transition of the problem extended to the early 2000s, with interesting applications in statistical physics \cite{mezard-information-2009, BorgsCP01, borgs-phase-2004}.

More recently, one of the results on the RSSP has attracted quite some attention.
A simplified statement for it would be
\begin{theorem}[Lueker, \cite{lueker-exponentially-1998}]\label{thm:lueker-original}
    Let $\rx_1, \dots, \rx_n$ be i.i.d. uniform random variables over $[-1, 1]$, and let $\varepsilon \in (0, 1)$.
    There exists a universal constant $C > 0$ such that, if $n \ge C\log_2 \frac{1}{\varepsilon}$, then, with high probability, for all $z \in [-1,1]$ there exists a subset $S_z \subseteq [n]$ for which
    \begin{equation*}
        \bigl\lvert z - \sum_{i \in S_z} \rx_i \bigr\rvert \le 2\varepsilon.
    \end{equation*}
\end{theorem}
That is, a rather small number (of the order of $\log \frac{1}{\varepsilon}$) of random variables suffices to have a high probability of approximating not only a single target $z$, but all values in an interval.
In fact, this result is asymptotically optimal, since each of the $2^n$ subsets can cover at most one of two values more than $2\varepsilon$ apart and, hence, we must have $n = \Omega(\log \frac{1}{\varepsilon})$.
Also, the original work generalises the result to a wide class of distributions.

Those features allowed \cref{thm:lueker-original} to be quite successful in applications.
In the field of Machine Learning, particularly, many recent works, such as \cite{PensiaetAl20, cunha-proving-2021, FischerB21, BurkholzLMG21, ferbach2022, WangDMWLLHMRD21}, leverage this result.
We discuss those contributions in more detail in \cref{sec:related}.

In this paper, we investigate a natural multidimensional generalisation of \cref{thm:lueker-original}. Mainly, we prove
\begin{theorem}[Main Theorem]
    \label{thm:multidim-rss-ours}
    Given $\varepsilon \in (0, 1)$ and
    $d, n \in \N$,  consider  $n$  independent  $d$-dimensional standard normal random vectors $\rvx_1, \dots, \rvx_n$.
    There exists a universal constant $C > 0$ for which, if
    $$
        n \ge C d^3 \log_2\frac{1}{\varepsilon} \cdot \left(\log_2\frac{1}{\varepsilon} + \log_2 d\right),
    $$
    then, with high probability, for all $\vz \in [-1,1]^d$ there exists a subset $S_{\vz} \subseteq [n]$ for which
    \begin{equation*}
        \bigl\lVert \vz - \sum_{i \in S_\vz} \rvx_i \bigr\rVert_\infty \le 2\varepsilon.
    \end{equation*}
    Moreover, the approximations can be achieved with subsets of size $\frac{n}{6\sqrt{d}}$.
\end{theorem}

We believe many promising applications of the RSSP can become feasible with this extension of \cref{thm:lueker-original} to multiple dimensions.
To illustrate this, we consider the \emph{Neural Network Evolution (NNE)} model recently introduced by \cite{gorantla-biologically-2019}.
It is natural to wonder whether their model is \emph{universal}, in the sense that,  with high probability, it can approximate any dense feed-forward neural network.
While applying \cref{thm:lueker-original} to this end would yield exponential bounds on the required overparameterization, in \cref{sec:NNE} we prove the universality of the model within polynomial bounds.
To broaden the scope of our result, we additionally provide some useful generalisations in \cref{sec:generalisation}.
In particular, we extend it to a wide class of distributions, proving an analogous extension to the one \citep{lueker-exponentially-1998} given for \cref{thm:lueker-original}.
Finally, in \cref{apx:discrete,apx:nondetrndwalks} we discuss a discretization of our result and potential applications in the context of nondeterministic random walks.

 \paragraph{Organisation of the paper.}
After discussing related works in \cref{sec:related},
we present a high level overview of the difficulties posed by the problem and of our proof of \cref{thm:multidim-rss-ours} (\cref{sec:overview}).
We then introduce our notation in \cref{sec:preliminaries} in preparation for the presentation of our analysis in \cref{sec:estimation-moments}.
We follow up with an application of our result to the NNE model \citep{gorantla-biologically-2019} and conclude with some notes on the tightness of our analysis in \cref{sec:lb-var}.
Finally,  generalisations of our results, further extensions, as well as all omitted proofs can be found in the Appendix.

 \section{Related work}
\label{sec:related}

As remarked in the Introduction, the first studies of the RSSP were mainly motivated by average-case analyses of the classic Subset Sum and Number Partition problems \cite{karmarkar1986, lueker1982, lueker-exponentially-1998}.
Both can be efficiently solved if the precision of the values considered is sufficiently low relative to the size of the input set.
In particular, \cite{Mertens1998PhaseTI} applies methods from statistical physics to indicate that this is a fundamental property of the problem: the amount of exact solutions of the randomised version exhibits a phase transition when the precision increases relative to the sample size.
The work \cite{BorgsCP01} later confirmed formally the existence of a phase transition.

The work of G.\ S.\ Lueker on the RSSP dates back to \cite{lueker1982}.
In \cite{lueker1982}, the author proves a weaker version of \cref{thm:lueker-original} and uses it as a tool to analyse the integrality gap of the one-dimensional integer Knapsack problem, i.e., the additive gap between the optimal solution of the integer problem and that of its linear programming relaxation, when the inputs are sampled according to some probability distribution.
Later, the same author provided a tighter result of the RSSP in \cite{lueker-exponentially-1998}, which we stated in \cref{thm:lueker-original}. 
Recently, \cite{da-cunha-revisiting-2022} exhibited a simpler alternative to the original proof.
\cite{dyer1989multidimKnapsack} generalized the result of the RSSP from \cite{lueker1982} to tackle the multidimensional formulation of the Knapsack problem.
In particular, it is proved that if the number of input variables is $\Theta(d \log \frac 1\varepsilon)$, then, with probability $e^{-\gO(d)}$, there exists a subset approximating a given target in $\R^d$.
Using the latter result as a black box, it is easy to see that one would require $ e^{\gO(d)} \log \frac 1\varepsilon $ input variables to increase the success probability to a constant value.
The result in \cite{dyer1989multidimKnapsack} has recently been improved by \cite{borst2022integrality-gap}, where tighter bounds on the integrality gap of the multidimensional Knapsack problem were obtained.
More specifically, \cite{borst2022integrality-gap} showed that, at the cost of an extra polynomial number of input random variables, the success probability of approximating a single target in the space can be increased to a constant value:
this probability is achieved whenever the number of variables $n$ satisfies the following relations: $n \ge d^{\frac 94}$ and $n = \Theta({d^{\frac 32} \log \frac 1 \varepsilon})$.
Both the analyses in \cite{dyer1989multidimKnapsack} and \cite{borst2022integrality-gap} employ the second moment method to estimate the probability that at least one subset  approximating the target value exists.
Following the same approach, in our work we refine the analysis for the second moment method technique. 
In order to prove \cref{thm:multidim-rss-ours}, we show that $n = \Theta(d^2 \log \frac 1\epsilon)$ variables yield constant probability that a subset approximating a single target exists.
Our bound is better than that in \cite{borst2022integrality-gap} for all approximation errors $\varepsilon$ which are not exponentially small in the dimension of the space, that is, $\varepsilon = {e^{-\gO(d^{\frac 34})}}$.
We also remark that the result in \cite{borst2022integrality-gap} is generalised to all distributions whose convergence to a Gaussian is ``fast enough'', which is a wider class of distribution with respect to the one we provide in this work.
Nevertheless, as we share the same approach of \cite{borst2022integrality-gap}, with the same arguments we can extend our results to a similar class of distributions.
In a recent follow-up \cite{borst2022discrepancy}, weaker bounds on the multidimensional RSSP are exhibited, which hold for an even wider class of distributions.
\footnote{\cite{borst2022integrality-gap} considers Gaussian or uniform input random variables, and extends its result to distributions that converge quickly to Gaussian ones.
In \cite{borst2022discrepancy}, the authors solve the RSSP for input random vectors whose entries follow uniform distributions on finite, discrete sets, and for input log-concave random vectors (i.e., when the density function of these vectors is a log-concave function).} 
In \cite{borst2022discrepancy}, the number of variables required to solve the problem is $\Theta(d^{6} \log \frac 1\varepsilon)$.
The discrete setting of a variant of the RSSP has also been recently studied in \cite{ChenJRS22} which proves that an integral linear combination (with coefficients in $\{-1, 0, 1\}$) of the sample variables can approximate a range of target values.

In the last few years, \cref{thm:lueker-original} has been very useful in studying the \emph{Strong Lottery Ticket Hypothesis}, which states that Artificial Neural Networks (ANN) with random weights are likely to
contain an approximation of any sufficiently smaller ANN as a subnetwork.
In particular, such claim poses the deletion of connections (pruning) as a theoretically solid alternative to careful calibration of their weights (training).
\cite{PensiaetAl20} uses \cref{thm:lueker-original} to prove the hypothesis under optimal overparameterization for dense ReLU neural networks. 
\cite{cunha-proving-2021} extends this result to convolutional networks and
\cite{ferbach2022} further extends the latter to the class of equivariant networks.
Also, \cite{BurkholzLMG21} applies \cref{thm:lueker-original} to construct neural networks that can be adapted to a variety of tasks with minimal retraining. \section{Overview of our analysis}
\label{sec:overview}

\subsection{Insights on the difficulty of the problem}
\label{sec:insights}

In $d$ dimensions,
since we need $2^{\Theta(d\log\frac 1\varepsilon)}$ hypercubes of radius $\varepsilon$ to cover the set $[-1,1]^d$,
we need a sample of $\Omega(d\log\frac 1\varepsilon)$ vectors to be able to approximate (up to error $\varepsilon$) every vector in $[-1, 1]^d$.

On the other hand, having $n = \bigO(d\log\frac 1\varepsilon)$ vectors is enough in expectation.
To see it, it is sufficient to consider subsets of the sample with $\frac n2$ vectors.
There are $\binom{n}{n/2} \approx 2^{n-o(n)}$ such subsets, each summing to a random vector distributed as $\NormalDist(\vzero, \frac n2 \cdot \mI_d)$.
Thus, given any $\vz \in [-1,1]^d$, each of those sums has probability approximately $\varepsilon^d (\frac n2)^{-\frac d2} = 2^{-d\log \frac 1\varepsilon -\frac d2 \log \frac n2}$ of being at most $\varepsilon$ far from $\vz$.
We can then conclude that the expected number of approximations is $2^{n-o(n)} \cdot 2^{-d\log\frac 1\varepsilon - \frac d2\log \frac n2}$, which is still of order $2^{n-o(n)}$ provided that $n \ge C d\log\frac 1\varepsilon$ for a sufficiently large constant $C$.

It would thus suffice to prove concentration bounds on the expectation.
The technical challenge is handling the stochastic dependency between subsets of the sample,
as pairs of those typically intersect, with many random variables thus appearing for both resulting sums.
The original proof of \cref{thm:lueker-original} \citep{lueker-exponentially-1998} and the simplified one \citep{da-cunha-revisiting-2022} address dependencies in similar ways.
Both keep track of the fraction of values in $[-1,1]$ that can be approximated by a sum of a subset of the first $i$ random variables, $\rx_1, \ldots, \rx_i$.
Their core goal is to bound the proportional increase in this fraction when an additional random variable $\rx_{i+1}$ is considered.
As it turns out, the \emph{conditional expectation} of this increment can be bounded by a constant factor, regardless of the values of $\rx_1, \ldots, \rx_i$.
Unfortunately, naively extending those ideas to $d$ dimensions leads to an estimation of this increment that is exponentially small in $d$.
It is not clear to the authors how to make the estimation depend polynomially on $d$ without leveraging some knowledge of the actual values of $\rx_1, \ldots, \rx_i$.
In fact, even which kind of assumption on the previous samples could work in this sense is not totally clear.

As for other classical concentration techniques that might appear suitable at first, we remark our failed attempts to leverage an average bounded differences argument \citep{warnke2016method}. 
Specifically, we could not identify any natural function related to the fraction of values that can be approximated, which was also Lipschitz relative to the sample vectors.
Moreover, both Janson's variant of Chernoff bound \citep{janson_large_2004} and a recent refinement of it \citep{wang_learning_2017} seem to capture the stochastic dependence of the subset sums too loosely for our needs.

\subsection{Our approach}
\label{sec:approach}

Our strategy to overcome the difficulties highlighted in the previous subsection consists in a second-moment approach.

Unlike the proofs for the single dimensional case, our argument, at first, analyses the probability of approximating a single target value $\vz \in [-1, 1]^d$.
To this end, consider a sample of $n$ independent random vectors $\rvx_1, \ldots, \rvx_n$ and a family $\CC$ of subsets of the sample.
Let $\ry$ be the number of subsets in $\CC$ whose sum approximates $\vz$ up to error $\varepsilon$.

For a single subset, it is not hard to estimate the probability with which a subset-sum $\sum_{i \in S} \rvx_i$ lies close to $\vz$.
This allows us to easily obtain good bounds on $\expect{\ry}$.

We, then, proceed to estimate the variance of $\ry$, circumventing the obstacles mentioned in the previous section by restricting the analysis to families of subsets with sufficiently small pairwise intersections.
While this restriction limits the maximum amount of subsets that are available, a standard probabilistic argument allows us to prove the existence of large families of subsets with the desired property, ensuring that $\expect{\ry}$ can be large enough for our purposes.

For each pair of subsets, $S$ and $T$, we leverage the hypothesis on the size of intersections to consider partitions $S = S_\ra \cup S_\rb$ and $T = T_\rc \cup T_\rb$, with
$S_\ra$ and $T_\rc$ being large, stochastically independent parts, and the smaller parts $S_\rb$ and $T_\rb$ containing $S \cap T$.
The bulk of our analysis then consists in deriving careful bounds on their reciprocal dependencies and consequent contributions to the second moment of $\ry$.

The resulting estimate allows us to apply Chebyshev's inequality to $\ry$, obtaining a constant lower bound on $\pr{\ry \ge 1}$.
That is, we conclude that with at least some constant probability at least one of the subsets yields a suitable approximation of $\vz$.
Finally, we employ a probability-amplification argument in order to apply a union bound over all possible target values in $[-1, 1]^d$.

 \section{Preliminaries}\label{sec:preliminaries}

\paragraph{Notation}
Throughout the text we identify the different types of objects by writing their symbols in different styles.
This applies to
scalars (e.g. $x$),
real random variables (e.g. $\rx$),
vectors (e.g. $\vx$),
random vectors (e.g. $\rvx$),
matrices (e.g. $\mX$).
and tensors (e.g. $\tX$).
In particular, for $d \in \N$, the symbol $\mI_d$ represents the $d$-dimensional identity matrix, where $\N$ refers to the set of positive integers.
Let $n \in \N$.
We denote the set $\{1, \dots, n\}$ by $[n]$,
and given a set $S$ employ the notation $\binom{S}{n}$ to refer to the family of all subsets of $S$ containing exactly $n$ elements of $S$.
Let $\vx \in \R^d$.
The notation $\lVert \vx \rVert_2$ represents the euclidean norm of $\vx$
while $\lVert \vx \rVert_{\infty}$ denotes its maximum-norm.
Moreover, given $r \in \R_{>0}$ we denote the set $\{\vy \in \R^d : \lVert \vy - \vx \rVert_\infty \le r\}$ by $\dBall(\vx, r)$.
We represent the variance of an arbitrary random variable $\rx$ by $\sigma_{\rx}^2$
and its density function by $\varphi_{\rx}$.
Finally, the notation $\log(\cdot)$ refers to the binary logarithm.
Let $d, n \in \N$ and $\varepsilon \in \R_{>0}$, and consider $\vz \in [-1, 1]^d$ and $n$ independent standard normal $d$-dimensional random vectors $\rvx_1, \dots, \rvx_n$.
Given $S \subseteq [n]$ we define the random variable
\begin{align*}
    \ry_{S, \varepsilon, \vz, \rvx_1, \dots, \rvx_n} =
    \begin{cases}
        1 &\textrm{ if } \lVert \vz - \sum_{i \in S} \rvx_i\rVert_\infty \le \varepsilon,\\
        0 &\textrm{ otherwise,}
    \end{cases}
\end{align*}
that we represent simply by $\ry_S$ when the other parameters are clear from context.
Since we are interested in studying families of subsets, we also define, for $\CC$ contained in the power set of $[n]$, the random variable
\begin{align*}
    \ry_{\CC, \varepsilon, \vz, \rvx_1, \dots, \rvx_n} = \sum_{S \in \CC} \ry_S,
\end{align*}
which we represent simply as $\ry$.

We control the stochastic dependency among subsets by restricting to families of subsets with small pairwise intersection.
While this reduces how many subsets we can be considered, we can use the probabilistic method to prove that large families are still available.
\begin{lemma}\label{lem:n-subsets}
    For all $n \in \N$ and $\alpha \in (0, \frac{1}{2})$,
    there exists $\CC \subseteq \binom{[n]}{\alpha n}$ with $\abs{\CC} \ge 2^{\frac{\alpha^2 n}{6}}$ such that for all $S, T \in \CC$, if $S \neq T$, then
    \begin{align*}
        \abs{S \cap T} \le 2\alpha^2 n.
    \end{align*}
\end{lemma}
Notice that, while this amount is still exponential, it already imposes $n = \Omega(\frac{d}{\alpha^2} \log \frac{1}{\varepsilon})$ if we are to approximate all points in $[-1, 1]^d$ up to error $\varepsilon$.
 \section{Proof of the main result}\label{sec:estimation-moments}

As we frequently consider values relatively close to the origin, approximation of the normal distribution by a uniform one is sufficient for many of our estimations.

\begin{lemma}\label{lem:uniform-approximation}
    Let $d \in \N$, $\varepsilon \in (0, 1)$, $\sigma \in \R_{>0}$, and $\vz \in [-1, 1]^d$.
    If $\rvx \sim \NormalDist(\vzero, \sigma^2\cdot\mI_d)$,
    then
    \begin{align*}
        e^{-\frac{2d}{\sigma^2}} \cdot \frac{(2\varepsilon)^d}{\left(2\pi\sigma^2\right)^{\frac{d}{2}}}
        \le \pr{\rvx \in \dBall(\vz,\varepsilon)}
        \le \frac{(2\varepsilon)^d}{\left(2\pi\sigma^2\right)^{\frac{d}{2}}}.
    \end{align*}
\end{lemma}

As a corollary, we bound the first moment of the random variable $\ry$.

\begin{corollary}\label{lem:expectation}
    Given $d, n \in \N$,
    $\varepsilon \in (0, 1)$, and
    $\alpha \in (0, \frac{1}{2})$,
    let $\rvx_1, \ldots, \rvx_n$ be independent standard normal $d$-dimensional random vectors.
    Then, for all $\vz \in [-1, 1]^d$ and
    $\CC \subseteq \binom{[n]}{\alpha n}$,
    it holds that
    \begin{align*}
        e^{-\frac{2d}{\alpha n}} \frac{(2\varepsilon)^d\abs{\CC}}{\left(2 \pi \alpha n\right)^{\frac{d}{2}}}
        \le \expect{\ry}
        \le \frac{(2\varepsilon)^d\abs{\CC}}{\left(2 \pi \alpha n\right)^{\frac{d}{2}}}.
    \end{align*}
\end{corollary}
\begin{proof}
Let $S \in \CC$ and, hence, $\abs{S} = \alpha n$.
Since $\rvx_i \sim \Normal{\vzero, \mI_d}$ for all $i \in [n]$, we have that $\sum_{i \in S} \rvx_i \sim \Normal{\vzero, \alpha n \cdot \mI_d}$.
Therefore, as
$
    \pr{\ry_S = 1}
    = \pr{\sum_{i \in S} \rvx_i \,\in \dBall(\vz,\varepsilon)},
$
by \cref{lem:uniform-approximation}, we have that
\begin{align*}
    e^{-\frac{2d}{\alpha n}} \frac{(2\varepsilon)^d}{\left(2 \pi \alpha n\right)^{\frac{d}{2}}}
    \le \pr{\ry_S = 1}
    \le \frac{(2\varepsilon)^d}{\left(2 \pi \alpha n\right)^{\frac{d}{2}}},
\end{align*}
and we can conclude the thesis by noting that $\expect{\ry} = \sum_{S \in \CC} \pr{\ry_S = 1}$.
\end{proof}
 We proceed by estimating the second moment of $\ry$.

\begin{lemma}\label{lem:variance}
    Given $d, n \in \N$,
    $\varepsilon \in (0, 1)$, and
    $\alpha \in (0, \frac{1}{6}]$,
    let $\rvx_1, \ldots, \rvx_n$ be independent $d$-dimensional standard normal random vectors,
    $\vz \in [-1, 1]^d$, and
    $\CC \subseteq \binom{[n]}{\alpha n}$.
If $n \ge \frac{81}{\alpha(1-2\alpha)}$
    and any two subsets in $\CC$ intersect in at most $2\alpha^2n$ elements,
    then
    \begin{align*}
        \var{\ry} \le \frac{(2\varepsilon)^{2d} \abs{\CC}^2}{(2 \pi \alpha n)^d} \cdot \left[(1 - 4\alpha^2)^{-\frac{d}{2}} - e^{-\frac{4d}{\alpha n}}\right] + \frac{(2\varepsilon)^d \abs{\CC}}{(2 \pi \alpha n)^{\frac{d}{2}}}.
    \end{align*}
\end{lemma}
\begin{proof}
We have
\begin{align*}
    \var{\ry}
    &= \sum_{S,T \in \CC} \cov{\ry_S,\ry_T} \nonumber
    \\&= \sum_{S,T \in \CC} \left(\expect{\ry_S \cdot \ry_T} - \expect{\ry_S} \expect{\ry_T}\right) \nonumber
    \\&= \sum_{S,T \in \CC} \left(\pr{\ry_S = 1, \ry_T = 1} - \pr{\ry_S = 1} \pr{\ry_T = 1}\right) \nonumber
    \\&= \sum_{S \neq T \in \CC} \left(\pr{\ry_S = 1, \ry_T = 1} - \pr{\ry_S = 1}^2\right) + \sum_{S\in \CC} \pr{\ry_S = 1}\left(1 - \pr{\ry_S = 1}\right). \end{align*}

We shall use \cref{lem:uniform-approximation} to estimate $\pr{\ry_S = 1}$,
thus, the core of our argument is to bound the joint probability $\pr{\ry_S = 1, \ry_T = 1}$.
To this end, since $\cov{\ry_S, \ry_T}$ increases monotonically with $\abs{S \cap T}$, we fix $S, T \in \CC$ with $\abs{S \cap T} = 2\alpha^2 n$.
Moreover, since $\ry_S$ is defined in terms of the max-norm, we can analyse the associated event for each coordinate independently.
So, we let $\rx_1, \cdots, \rx_n \sim \NormalDist(0, 1)$
and $z \in [-1,1]$.

Consider the partitions $S = S_\ra \cup S_\rb$ and $T = T_\rc \cup T_\rb$, with
$
    S_\rb = T_\rb = S \cap T,
$
and let
\begin{align*}
    \ra = \sum_{i \in S_\ra} \rx_i,\qquad
    \rc = \sum_{i \in T_\rc} \rx_i,\qquad
    \rb = \sum_{i \in S \cap T} \rx_i.
\end{align*}
In this way, we have
$\sum_{i \in S} \rx_i = \ra + \rb$ and
$\sum_{i \in T} \rx_i = \rc + \rb$,
with $\ra, \rc$ independent random variables distributed as $\NormalDist(0, \sigma_\ra^2)$ and
$\rb \sim \NormalDist(0, \sigma_\rb^2)$,
where $\sigma_\ra^2 = \alpha n(1 - 2\alpha)$ and $\sigma_\rb^2 = 2\alpha^2 n$.

With this setup, we have,
\begin{align*}
    \pr{\ry_S = 1, \ry_T = 1} &= \bigl(\pr{\ra+\rb \in (z-\varepsilon, z+\varepsilon),\, \rc+\rb \in (z-\varepsilon, z+\varepsilon)}\bigr)^d.
\end{align*}
From the law of total probability, it holds that
\begin{align}
    &\pr{\ra+\rb \in (z-\varepsilon, z+\varepsilon),\, \rc+\rb \in (z-\varepsilon, z+\varepsilon)} \nonumber
    \\&\qquad= \int_{\R} \varphi_\rb(x) \cdot \pr{\ra+x \in (z-\varepsilon, z+\varepsilon),\, \rc+x \in (z-\varepsilon, z+\varepsilon)} \,\diff x \nonumber
    \\&\qquad= \int_{\R} \varphi_\rb(x) \cdot \pr{\ra \in (z-x-\varepsilon, z-x+\varepsilon),\, \rc \in (z-x-\varepsilon, z-x+\varepsilon)} \,\diff x \nonumber
    \\&\qquad= \int_{\R} \varphi_\rb(x) \cdot \left(\pr{\ra \in (z-x-\varepsilon, z-x+\varepsilon)}\right)^2 \,\diff x, \label{eq:last-equality}
\end{align}
where the last equality follows from the independence of $\ra$ and $\rc$.

Since $\ra$ is a normal random variable with 0 average,
by \cref{claim:int-ub-max}, we have that
\begin{align*}
    \int_{\R} \varphi_\rb(x) \cdot \left(\pr{\ra \in (z-x-\varepsilon, z-x+\varepsilon)}\right)^2 \,\diff x
    &\le \int_{\R} \varphi_\rb(x) \cdot \left(\pr{\ra \in (x-\varepsilon, x+\varepsilon)}\right)^2 \,\diff x  \nonumber
    \\&= \int_{\R} \varphi_\rb(x) \cdot \left(\int_{x-\varepsilon}^{x+\varepsilon} \varphi_\ra(y) \,\diff y\right)^2 \,\diff x. \end{align*}

The hypothesis on $n$ implies that $2\sigma_a^2 \ge 162$,
so, by \cref{claim:int-ub-convex},
\begin{align*}
    \left(\int_{x-\varepsilon}^{x+\varepsilon} \varphi_\ra(y) \,\diff y\right)^2
    &\le \biggl[\int_{x-\varepsilon}^{x+\varepsilon} \frac{\exp\left(-\frac{(x+\varepsilon)^2}{2\sigma_\ra^2}\right) + \exp\left(-\frac{(x-\varepsilon)^2}{2\sigma_\ra^2}\right)}{2\sqrt{2\pi\sigma_\ra^2}} \cdot \exp\Bigl(\frac{\varepsilon^2}{2\sigma_\ra^2}\Bigr) \, \diff y\biggr]^2
    \\&= \frac{(2\varepsilon)^2}{2\pi\sigma_\ra^2} \cdot \frac{\exp\left(-\frac{(x+\varepsilon)^2}{\sigma_\ra^2}\right) + \exp\left(-\frac{(x-\varepsilon)^2}{\sigma_\ra^2}\right) + 2\exp\left(-\frac{x^2+\varepsilon^2}{\sigma_\ra^2}\right)}{4} \cdot \exp\left(\frac{\varepsilon^2}{\sigma_\ra^2}\right)
    \\&= e^{\varepsilon^2/\sigma_\ra^2} \cdot \frac{1}{\sqrt{2}} \cdot \frac{(2\varepsilon)^2}{\sqrt{2 \pi \sigma_\ra^2}} \cdot \frac{\varphi_{\ra/\sqrt{2}}(x+\varepsilon) + \varphi_{\ra/\sqrt{2}}(x-\varepsilon) + 2e^{-\varepsilon^2/\sigma_\ra^2} \cdot \varphi_{\ra/\sqrt{2}}(x)}{4}.
\end{align*}
Moreover, it holds that
\begin{align*}
    &\int_{\R} \varphi_\rb(x) \cdot \left[\varphi_{\ra/\sqrt{2}}(x+\varepsilon) + \varphi_{\ra/\sqrt{2}}(x-\varepsilon) + 2e^{-\varepsilon^2/\sigma_\ra^2} \cdot \varphi_{\ra/\sqrt{2}}(x)\right] \,\diff x
    \\&\qquad= (\varphi_\rb * \varphi_{\ra/\sqrt{2}})(\varepsilon) + (\varphi_\rb * \varphi_{\ra/\sqrt{2}})(-\varepsilon) + 2e^{-\varepsilon^2/\sigma_\ra^2} \cdot (\varphi_\rb * \varphi_{\ra/\sqrt{2}})(0)
    \\&\qquad= \varphi_{\rb+\ra/\sqrt{2}}(\varepsilon) + \varphi_{\rb+\ra/\sqrt{2}}(-\varepsilon) + 2e^{-\varepsilon^2/\sigma_\ra^2} \cdot \varphi_{\rb+\ra/\sqrt{2}}(0)
    \\&\qquad= \frac{2e^{-\varepsilon^2/\sigma^2_{\rb+\ra/\sqrt{2}}} + 2e^{-\varepsilon^2/\sigma_\ra^2}}{\sqrt{2\pi\sigma^2_{\rb+\ra/\sqrt{2}}}}
    \\&\qquad\le 4 \cdot \frac{e^{-\varepsilon^2/\sigma_\ra^2}}{\sqrt{2\pi\sigma^2_{\rb+\ra/\sqrt{2}}}},
\end{align*}
here $*$ denotes the convolution operation, and the last inequality comes from the hypothesis $\alpha \le \frac{1}{6}$, which implies that $\sigma^2_{\rb+\ra/\sqrt{2}} \le \sigma_\ra^2$.

Altogether, we have
\begin{align}
    \label{eq:upper-on-joint}
    \pr{\ry_S = 1, \ry_T = 1}
    &\le \left(e^{\varepsilon^2/\sigma_\ra^2} \cdot \frac{1}{\sqrt{2}} \cdot \frac{(2\varepsilon)^2}{\sqrt{2 \pi \sigma_\ra^2}} \cdot \frac{e^{-\varepsilon^2/\sigma_\ra^2}}{\sqrt{2\pi\sigma^2_{\rb+\ra/\sqrt{2}}}}\right)^d
    \\&= \left(\frac{(2\varepsilon)^2}{2 \pi} \cdot \frac{1}{\sqrt{2\sigma_\ra^2\sigma_{\rb+\ra/\sqrt{2}}^2}}\right)^d \nonumber
    \\&= \frac{(2\varepsilon)^{2d}}{(2 \pi \alpha n)^d} \cdot (1 - 4\alpha^2)^{-\frac{d}{2}}, \nonumber \end{align}
where the last equality follows from recalling that $\sigma_\rb^2 = 2\alpha^2 n$ and $\sigma_\ra^2 = \alpha n\left(1 - 2\alpha\right)$, and, thus, $\sigma_{\rb+\ra/\sqrt{2}}^2 = 2\alpha^2 n + \frac{\alpha t}{2}\left(1 - 2\alpha\right)$.

Finally, from this bound and from \cref{lem:uniform-approximation} we can conclude that
\begin{align*}
    \var{\ry} &= \sum_{S \neq T \in \CC} \left(\pr{\ry_S = 1, \ry_T = 1} - \pr{\ry_S = 1}^2\right) + \sum_{S\in \CC} \pr{\ry_S = 1}\left(1 - \pr{\ry_S = 1}\right)
    \\&\le \sum_{S\neq T \in \CC} \Bigl[\frac{(2\varepsilon)^{2d}}{(2 \pi \alpha n)^d} \cdot (1 - 4\alpha^2)^{-\frac{d}{2}} - \frac{(2\varepsilon)^{2d}}{(2 \pi \alpha n)^d} \cdot e^{-\frac{4d}{\alpha n}}\Bigr]
    + \sum_{S \in \CC} \frac{(2\varepsilon)^d}{\left(2 \pi \alpha n\right)^{\frac{d}{2}}} \Bigl[1 - e^{-\frac{2d}{\alpha n}} \cdot \frac{(2\varepsilon)^d}{\left(2 \pi \alpha n\right)^{\frac{d}{2}}}\Bigr]
    \\&\le \frac{(2\varepsilon)^{2d} \abs{\CC}^2}{(2 \pi \alpha n)^d} \cdot \left[(1 - 4\alpha^2)^{-\frac{d}{2}} - e^{-\frac{4d}{\alpha n}}\right]
    + \frac{(2\varepsilon)^d \abs{\CC}}{(2 \pi \alpha n)^{\frac{d}{2}}}.
\end{align*}
\end{proof}

\begin{remark}
    In the proof of \cref{lem:variance}, after applying the law of total probability it is possible to employ \cref{lem:uniform-approximation} to estimate the joint probability.
    While this simplifies the argument, doing so would ultimately weaken the bound in \cref{thm:target} by a factor of $d$.
    In fact, in \cref{sec:lb-var} we argue that the estimation we provide is essentially optimal.
\end{remark}
 For our next result, recall that the existence of a suitable family of subsets is ensured by \cref{lem:n-subsets}.

\begin{lemma}\label{lem:chebyshev}
    Given $d, n \in \N$,
    $\varepsilon \in (0, 1)$, and
    $\alpha \in (0, \frac{1}{6}]$,
    let $\rvx_1, \ldots, \rvx_n$ be independent $d$-dimensional standard normal random vectors,
    $\vz \in [-1, 1]^d$, and
    $\CC \subseteq \binom{[n]}{\alpha n}$
    with $\abs{\CC} \ge 2^{\frac{\alpha^2n}{6}}$.
If any two subsets in $\CC$ intersect in at most $2\alpha^2n$ elements,
    $\alpha \le \frac{1}{6\sqrt{d}}$, and
    $$n \ge \frac{144d}{\alpha^2}\left(\log\frac{1}{\varepsilon} + \log d + \log\frac{1}{\alpha}\right),$$
    then
    \begin{align*}
        \pr{\ry \ge 1} \ge \frac{1}{3}.
    \end{align*}
\end{lemma}
\begin{proof}
By Chebyshev's inequality, it holds that
\begin{align*}
    \pr{\ry \ge 1} & \ge \pr{\abs{\ry - \expect{\ry}} < \frac{\expect{\ry}}{2}} \\
    & \ge 1 - \frac{4 \cdot \var{\ry}}{\expect{\ry}^2}.
\end{align*}

Applying \cref{lem:expectation,lem:variance}, we get that
\begin{align*}
    \frac{4 \cdot \var{\ry}}{\expect{\ry}^2}
    &\le 4 \cdot \frac{e^{\frac{4d}{\alpha n}} \cdot \left(2 \pi \alpha n\right)^{d}}{(2\varepsilon)^{2d} \abs{\CC}^2} \cdot \left[\frac{(2\varepsilon)^{2d} \abs{\CC}^2}{(2 \pi \alpha n)^d} \cdot \left[(1 - 4\alpha^2)^{-\frac{d}{2}} - e^{-\frac{4d}{\alpha n}}\right] + \frac{(2\varepsilon)^d \abs{\CC}}{(2 \pi \alpha n)^{\frac{d}{2}}}\right]
    \\&= 4 \cdot \left(\frac{e^{\frac{4d}{\alpha n}}}{(1 - 4\alpha^2)^{\frac{d}{2}}} - 1\right) + \frac{4e^{\frac{4d}{\alpha n}} \cdot (2 \pi \alpha n)^{\frac{d}{2}}}{(2\varepsilon)^d \abs{\CC}}.
\end{align*}

Since $n \ge \frac{68d}{\alpha}$ and $\alpha \le \frac{1}{6\sqrt{d}}$, by \cref{claim:ub-cov-term}
\begin{align*}
    4 \cdot \left(\frac{e^{\frac{4d}{\alpha n}}}{(1 - 4\alpha^2)^{\frac{d}{2}}} - 1\right)
    \le \frac{1}{2}.
\end{align*}
Furthermore, as $n \ge \frac{144d}{\alpha^2}\left(\log\frac{1}{\varepsilon} + \log d + \log\frac{1}{\alpha}\right)$, $\abs{\CC} \ge 2^{\frac{\alpha^2n}{6}}$, and $\alpha \le \frac{1}{6}$, by \cref{claim:ub-var-term},
\begin{align*}
    \frac{4e^{\frac{4d}{\alpha n}} \cdot (2 \pi \alpha n)^{\frac{d}{2}}}{(2\varepsilon)^d \abs{\CC}}
    \le \varepsilon.
\end{align*}
\end{proof}

Applying an union bound, we amplify the last lemma to get our main result.

\begin{theorem}\label{thm:target}
    Let $\varepsilon \in (0, 1)$ and
    given $d, n \in \N$ let $\rvx_1, \dots, \rvx_n$ be independent standard normal $d$-dimensional random vectors
    and let $\alpha \in \bigl(0, \frac{1}{6\sqrt{d}}\bigr]$.
    There exists a universal constant $C > 0$ such that, if
    $$
        n \ge C \frac{d^2}{\alpha^2} \log\frac{1}{\varepsilon} \cdot \left(\log\frac{1}{\varepsilon} + \log d + \log\frac{1}{\alpha}\right),
    $$
    then,
    with probability
    \begin{align*}
        1 - \exp \left[{- \ln 2 \cdot \left(\frac{n }{C \frac{d}{\alpha^2}\left(\log \frac{1}{\varepsilon} + \log d + \log \frac{1}{\alpha}\right) } - d \log \frac{1}{\varepsilon}\right)}\right],
    \end{align*}
    for all $\vz \in [-1,1]^d$ there exists a subset $S_{\vz} \subseteq [n]$ for which
    \begin{equation*}
        \bigl\lVert \vz - \sum_{i \in S_\vz} \rvx_i \bigr\rVert_\infty \le 2\varepsilon.
    \end{equation*}
    Moreover, this remains true even when restricted to subsets of size $\alpha n$.
\end{theorem}

\cref{thm:multidim-rss-ours} follows from \cref{thm:target} by setting $\alpha = \frac{1}{6\sqrt{d}}$.

\section{Application to Neural Net Evolution} \label{sec:NNE}
In this section, we present an application of our main result on the multidimensional RSSP (see Theorem~\ref{thm:multidim-rss-ours}) to a neural network model recently introduced in \cite{gorantla-biologically-2019}.

We first provide a description of their model in a setting relevant to our application.
Then, we prove that their model exhibits \emph{universality}; namely, with high probability, it can approximate any neural network within a polynomial overhead in the number of parameters.

\subsection{The NNE model}

The \emph{Neural Net Evolution} (NNE) model \cite{gorantla-biologically-2019} has been recently introduced as an alternative approach to train neural networks, based on evolutionary methods.
The aim is to provide a biologically inspired alternative to the backpropagation process behind ANNs \cite{rumelhart1986learning, goodfellow2016deep}, which happens in evolutionary time, instead of lifetime.

The NNE model is inspired by a standard update rule in population genetics and, in \cite{gorantla-biologically-2019},
it is shown to succeed in creating neural networks that can learn linear classification problems reasonably well with no explicit backpropagation.

To define the NNE model, we first need to define random genotypes. Given a vector $\vp \in [0,1]^n$,  a random  \emph{genotype} $\vx \in \{0,1\}^n$ is sampled by setting $x_i = 1$ with probability $p_i$, independently for each $i$.
Each entry $x_i$ indicates whether or not a \emph{gene} is active.

Then, for each $i$, a random tensor $\mathbf{\Theta}^{(i)} \in \R^{\ell \times d \times d}$ is sampled.
In the original version of the model \cite{gorantla-biologically-2019}, each entry of the tensor is chosen independently and uniformly at random from $[-1,1]$ with probability $\beta$, while it is set to 0, otherwise.
For the sake of our application, we here consider a natural variant where the entries of the tensor are independently drawn from a standard normal distribution.

Now, given a genotype $\vx \in \{0,1\}^n$, we define
\begin{align}
    \mathbf{\Theta}_{\vx} = \sum_{i \: : \, x_i = 1} \mathbf{\Theta}^{(i)}.\label{eq:genotype-tensor}
\end{align}
Each genotype is then associated with a \emph{feed-forward neural network}, represented by a weighted complete multipartite directed graph. The graph is formed by layers $\{L_i\}_{i=0}^\ell$ of $d$ nodes and two consecutive layers are fully connected via a biclique whose edge weights are determined by the  tensor $\mathbf{\Theta}_{\vx}$ in the following manner: for every $i \in [\ell]$, the edge between the $j$-th node of layer $L_{i-1}$ and the $k$-th node of layer $L_{i}$ has weight $(\mathbf{\Theta}_{\vx})_{ijk}$.

\cref{eq:genotype-tensor} tells us that if a gene is active then it gives a random contribution to each weight of the genotype network.

The learning process in  the NNE model works  by updating the genotype probabilities $\vp$ according to  some standard population genetics equations \cite{burger2000mathematical, chastain2014algorithms}.
In \cite{gorantla-biologically-2019}, it is proved that the adopted  update rule indirectly performs backpropagation and enables to decrease the loss function of the networks.

\subsection{Universality and RSSP}
Let $f \colon \R^{d} \to \R^{d}$ be a feed-forward neural network of the form
\begin{align}
    f(\vy) = \tW_{\ell} \,\sigma(\tW_{\ell-1}\dots \sigma(\tW_1\,\vy)),\label{eq:neural-network}
\end{align}
where $\tW_i \in \R^{d \times d}$ is a weight matrix and $\sigma \colon \R^d \to \R^d$ is the ReLU (Rectified Linear Unit) activation function that converts each coordinate $y_i$ of a given vector $\vy \in \R^d$ to $\max(0, y_i)$.

The restrictions on the weight matrix sizes $d \times d$ aim only to ease presentation and can be adapted to any arbitrary dimensions.

Let us construct a third-order tensor $\mathbf{\Theta}_f \in \R^{\ell \times d \times d}$ by stacking the weight matrices $\tW_1,\dots,\tW_{\ell}$. We correspondingly denote $f$ by $f_{\mathbf{\Theta}}$. Conversely, every tensor $\mathbf{\Theta} \in \R^{\ell \times d \times d}$ is associated with a neural network $f_{\mathbf{\Theta}}$ in the form of \cref{eq:neural-network} whose corresponding weight matrices are the tensor slices, that is, $\tW_m = (\mathbf{\Theta})_{\underset{j,k \in [d]}{i=m}}$ for every $m \in [\ell]$.

We can use \cref{thm:multidim-rss-ours} to prove a notion of universality for the NNE model.
\begin{theorem}\label{approx-tensor}
Let $\varepsilon > 0$ and $n, d, \ell \in \N$.
Let $\mathcal F$ be the class of neural networks $f \colon \R^{d} \to  \R^{d}$ of the form given in \cref{eq:neural-network} such that their corresponding tensor satisfies $\max_{ijk}\lvert(\mathbf{\Theta}_{f})_{ijk}\rvert < 1$.
A constant $C > 0$ exists such that, if $n \ge C (\ell \cdot d \cdot d)^3 \log\frac{1}{\varepsilon} \cdot \left(\log\frac{1}{\varepsilon} + \log(\ell \cdot d \cdot d)\right)$,
then, with high probability, the tensors $\mathbf{\Theta}^{(1)},\dots,\mathbf{\Theta}^{(n)}$ associated to each gene are such that,
for any $f\in \mathcal F$,
there is a genotype $\vx \in \{0,1\}^n$ which satisfies
\begin{align*}
    \max_{\underset{j,k \in [d]}{i \in [\ell]}} \abs{(\mathbf{\Theta}_f)_{ijk} - (\mathbf{\Theta}_{\vx})_{ijk}} <2\varepsilon.
\end{align*}
\end{theorem}

We note that standard techniques (e.g., \cite{PensiaetAl20,cunha-proving-2021}) can be used to provide bounds on the approximation of the output of neural networks, as well as translating \cref{approx-tensor} for general network architectures (e.g., convolutional neural networks).
 \section{Tightness of analysis}\label{sec:lb-var}

In Lemma \ref{lem:n-subsets} we prove the existence of a suitable family of subsets via a probabilistic argument,
sampling their elements uniformly at random.
The same argument also implies that the pairwise intersections of almost all subsets is at least $\frac{\alpha^2n}{2}$.
In the next result, we assume such lower bound and prove that our estimation of the joint probability $\pr{\ry_S = 1, \ry_T = 1}$ in \cref{lem:variance} (specifically, in Eq. \ref{eq:upper-on-joint}), is essentially tight.
Namely, the next lemma implies that it is not possible to obtain a high-probability bound on $\ry$ in \cref{lem:chebyshev}.

\begin{lemma}
\label{lem:tightness:covariance}
    Given $d, n \in \N$,
    $\varepsilon \in (0, 1)$, and
    $\alpha \in (0, \frac{1}{2})$,
    let $\rvx_1, \ldots, \rvx_n$ be independent standard normal $d$-dimensional random vectors
    and
    $\vz \in [-1, 1]^d$.
    If any two subsets in $\CC$ intersect in at least $\frac{\alpha^2n}{2}$ elements and
    $n \ge \frac{10}{\alpha(2-\alpha)}$, then
\begin{align*}
        \pr{\ry_S = 1, \ry_T = 1}
        &\ge \frac{(2\varepsilon)^{2d}}{(2 \pi \alpha n)^d} \cdot \left(1 - \frac{\alpha^2}{4}\right)^{-\frac{d}{2}} \cdot \exp\left(-\frac{3d}{\alpha n}\right).
    \end{align*}
\end{lemma}

We can extend the above result by letting $\vz$ lie in a wider range. This will be useful for the generalisation section \cref{sec:generalisation}.
\begin{remark}\label{remark:tightness:generalised}
    If $\lambda > 1$ and $\vz \in [-\lambda \sqrt{n}, \lambda \sqrt{n}]^d$, then we have
    \[
        \pr{\ry_S = 1, \ry_T = 1}
        \ge \frac{(2\varepsilon)^{2d}}{(2 \pi \alpha n)^d} \cdot \left(1 - \frac{\alpha^2}{4}\right)^{-\frac{d}{2}} \cdot \exp\left(-\frac{3\lambda^2 d}{\alpha}\right).
    \]
\end{remark}
 \section{Acknowledgements}

We would like to thank Bianca C. Araújo and Paulo B. S. Serafim for their feedback on preliminary versions of the manuscript, and Michele Salvi for the helpful discussions about the problem.

This work has been supported by the AID INRIA-DGA agreement n°2019650072,
by the Visiting 2021 funding (Decreto n. 1589/2021 2/7/2021) of University of Rome "Tor Vergata", and
partially supported by the ERC Advanced Grant 788893 AMDROMA ``Algorithmic and Mechanism Design Research in Online Markets'', the EC H2020RIA project ``SoBigData++'' (871042), the MIUR PRIN project ALGADIMAR ``Algorithms, Games, and Digital Markets, and the French government through the UCA JEDI (ANR-15-IDEX-01) and EUR DS4H (ANR-17-EURE-004) Investments in the Future projects. 
\bibliography{bibliography}

\appendix
\section{Tools}\label{sec:tools}
Below we list some standard tools we use, and prove some inequalities.

\subsection{Concentration bounds}
\begin{theorem}[Chebyshev's inequality]\label{eq:cheby}
Let $\rx$ be a random variable with finite expected value $\mu$ and finite non-zero variance $\sigma^2$. Then for any real number $k > 0,$ it holds that
\begin{align*}
    \pr{\abs{\rx-\mu} \ge k} \le \frac{\sigma^2}{k^2}.
\end{align*}
\end{theorem}

\begin{lemma}[Chernoff-Hoeffding bounds \cite{doerr2011hypergeometric}]\label{lemma:chernoff-hoeffding}
    Let $\rx_1, \rx_2, \dots, \rx_n$ be independent random variables such that $$\pr{0 \le \rx_i \le 1} = 1$$ for all $i \in [n]$. Let $\rx = \sum_{i=1}^n \rx_i$ and $\expect{\rx} = \mu$. Then, for any $\delta \in (0,1)$ the following holds:
    \begin{enumerate}
        \item
$
		    \pr{\rx \ge (1 + \delta)\mu} \le \exp\left (- \frac{\delta^2 \mu_+}{3}\right );
		$
		\item
$
		    \pr{\rx \le (1 - \delta)\mu} \le \exp\left (- \frac{\delta^2 \mu_+}{2}\right ).
		$
    \end{enumerate}

\end{lemma}

\subsection{Claims}
\begin{claim}\label{claim:ub-cov-term}
    Let $d, n \in \N$ and $\alpha \in \R_{>0}$.
    If $n \ge \frac{68d}{\alpha}$ and $\alpha \le \frac{1}{6\sqrt{d}}$, then
    \begin{align*}
        e^{\frac{4d}{\alpha n}} \cdot \frac{1}{\left(1 - 4\alpha^2\right)^{\frac{d}{2}}}
        &\le 1 + \frac{1}{8}.
    \end{align*}
\end{claim}
\begin{proof}
Since $e^x \le (1 - x)^{-1}$ for $x \le 1$, for $n \ge \frac{4d}{\alpha}$, it holds that
\begin{align*}
    e^{\frac{4d}{\alpha n}}
    \le \frac{1}{1 - \frac{4d}{\alpha n}}
    = 1 + \frac{4d}{\alpha n - 4d}.
\end{align*}
Thus, having $n \ge \frac{68d}{\alpha}$ implies that
\begin{align*}
    e^{\frac{4d}{\alpha n}} &\le 1 + \frac{1}{16}.
\end{align*}

Moreover, by Bernoulli's inequality,
since $\alpha < \frac{1}{2}$, it holds that,
\begin{align*}
    \frac{1}{\left(1 - 4\alpha^2\right)^{\frac{d}{2}}} &\le \frac{1}{1 - 2d\alpha^2}.
\end{align*}

Altogether, we need that
\begin{align*}
    \frac{1 + \frac{1}{16}}{1 - 2d\alpha^2} &\le 1 + \frac{1}{8},
\end{align*}
which holds for $\alpha \le \frac{1}{6\sqrt{d}}$.
\end{proof}

\begin{claim}\label{claim:ub-var-term}
    Let $d, n \in \N$, $\varepsilon \in (0,1)$, and $ \alpha \in (0, \frac{1}{6})$.
    If $n \ge \frac{144d}{\alpha^2} \left(\log\frac{1}{\varepsilon} + \log d + \log\frac{1}{\alpha}\right)$,
then
    \begin{align*}
        \frac{4e^{\frac{4d}{\alpha n}}}{2^{\frac{\alpha^2 n}{6}}} \cdot \left(\frac{\pi \alpha n}{2\varepsilon^2}\right)^{\frac{d}{2}}
        \le \varepsilon.
    \end{align*}
\end{claim}
\begin{proof}
Consider the function
\begin{align*}
    f(n) = \frac{n^{d}}{2^{\frac{\alpha^2 n}{6}}}.
\end{align*}
We have that
\begin{align*}
    f'(n)
    &= \frac{dn^{d-1} 2^{\frac{\alpha^2 n}{6}} - \frac{\alpha^2 \ln 2}{6} \cdot n^d 2^{\frac{\alpha^2 n}{6}}}{2^{\frac{\alpha^2 n}{3}}}
    \\&= \frac{n^{d-1} 2^{\frac{\alpha^2 n}{6}}}{2^{\frac{\alpha^2 n}{3}}} \cdot \left(d - \frac{\alpha^2 n \ln 2}{6}\right),
\end{align*}
and, hence, $f$ is non-increasing for $n \ge \frac{6d}{\alpha^2 \ln 2}$.
Thus, since
$
    f\left(\frac{6d}{\alpha^2 \ln 2}\right)
    = \left(\frac{6d}{e\alpha^2 \ln 2}\right)^d,
$
it holds that
\begin{align*}
    \frac{4e^{\frac{4d}{\alpha n}}}{2^{\frac{\alpha^2 n}{6}}} \cdot \left(\frac{\pi \alpha n}{2\varepsilon^2}\right)^{\frac{d}{2}}
    &= \frac{4e^{\frac{4d}{\alpha n}}}{2^{\frac{\alpha^2 n}{12}}} \cdot \left(\frac{\pi \alpha}{2\varepsilon^2}\right)^{\frac{d}{2}} \sqrt{\frac{n^d}{2^{\frac{\alpha^2 n}{6}}}}
    \\&\le \frac{4e^{\frac{4d}{\alpha n}}}{2^{\frac{\alpha^2 n}{12}}} \cdot \left(\frac{\pi \alpha}{2\varepsilon^2}\right)^{\frac{d}{2}} \sqrt{\left(\frac{6d}{e \alpha^2 \ln 2}\right)^d}
    \\&= \frac{4e^{\frac{4d}{\alpha n}}}{2^{\frac{\alpha^2 n}{12}}} \cdot \left(\frac{3 \pi d}{\varepsilon^2 e \alpha \ln 2}\right)^{\frac{d}{2}}
    \\&< \frac{8}{2^{\frac{\alpha^2 n}{12}}} \cdot \left(\frac{6d}{\varepsilon^2\alpha}\right)^{\frac{d}{2}}
\end{align*}
where the last inequality comes from noting that $6 > \frac{3\pi}{e\ln 2}$ and that $n \ge \frac{8d}{\alpha}$ implies $e^{\frac{4d}{\alpha n}} < 2$.
This is at most $\varepsilon$ if
\begin{align*}
    \frac{8}{\varepsilon^{d+1}} \cdot \left(\frac{6d}{\alpha}\right)^{\frac{d}{2}} &\le 2^{\frac{\alpha^2 n}{12}},
\end{align*}
or, equivalently,
\begin{align*}
    n
    &\ge \frac{12}{\alpha^2} \left(\log 8 + (d+1)\log \frac{1}{\varepsilon} + \frac{d}{2} \log \frac{1}{\alpha} + \frac{d}{2}\log 6d\right).
\end{align*}
The thesis follows from the bounds $d, n \ge 1$, $\varepsilon \in (0, 1)$, and $\alpha < \frac{1}{6}$.
\end{proof}

\begin{claim}\label{claim:int-ub-max}
Let $\ra,\rb$ be two centred normal random variables, and let $\varphi_\rb(x)$ be the density function of $\rb$. Then, for any $z \in \R$, for any $\varepsilon > 0$, it holds that
\begin{align*}
    \int_\R \varphi_\rb(x) \left[\pr{\ra \in (z - x - \varepsilon, z - x + \varepsilon)}\right]^2\,\diff x \le \int_\R \varphi_\rb(x) \left[\pr{\ra \in ( - x - \varepsilon, - x + \varepsilon)}\right]^2\,\diff x .
\end{align*}
\end{claim}
\begin{proof}

    For any $x,z \in \R$, let
    \[
        h(x,z) = \varphi_\rb(x) \left[\pr{\ra \in (z - x - \varepsilon, z - x + \varepsilon)}\right]^2\,\diff x,
    \]
    and let
    \[
        H(z) = \int_\R h(x,z) \, \diff x.
    \]
    Let $\varphi_\ra(x)$ be the density function of $a$.
    Since
    \begin{align*}
        \abs{\frac{\partial h(x,z)}{ \partial z}} & = 2\abs{\varphi_\rb(x) \pr{\ra \in (z - x - \varepsilon, z - x + \varepsilon)} \left(\varphi_\ra(z - x + \varepsilon) - \varphi_\ra(z - x - \varepsilon)\right)} \\
        & \le 2 \varphi_\rb(x) \pr{\ra \in (z - x - \varepsilon, z - x + \varepsilon)} \left(\varphi_\ra(z - x + \varepsilon) + \varphi_\ra(z - x - \varepsilon)\right),
    \end{align*}
    $h(x,z)$ meets the hypothesis of the Leibniz integral rule and we can write
    \begin{align*}
        \frac{\diff H(z)}{ \diff z}  & =
        \int_\R \frac{\partial h(x,z)}{\partial z} \, \diff x \\
        & = 2 \int_\R \varphi_\rb(x) \pr{\ra \in (z - x - \varepsilon, z - x + \varepsilon)} \left(\varphi_\ra(z - x + \varepsilon) - \varphi_\ra(z - x - \varepsilon)\right) \, \diff x
    \end{align*}
    If we prove that such a function is zero in $z = 0$, positive for $z < 0$ and negative for $z > 0$, then we have that the maximum of $H$ is reached in $z= 0$.

    \paragraph{First case: $z = 0$.} Then
    \begin{align}
        \frac{\diff H(0)}{ \diff z}  & = 2 \int_\R \varphi_\rb(x) \pr{\ra \in (x - \varepsilon, x + \varepsilon)} \left(\varphi_\ra( x - \varepsilon) - \varphi_\ra( x + \varepsilon)\right) \, \diff x \label{eq:claimIntMaxTrick1} \\ \nonumber
        & = 2 \int_\R \varphi_\rb(x) \pr{\ra \in (x - \varepsilon, x + \varepsilon)}\varphi_\ra( x - \varepsilon) \, \diff x \\ \nonumber
        & \ \ \ \ - 2 \int_\R \varphi_\rb(x) \pr{\ra \in (x - \varepsilon, x + \varepsilon)} \varphi_\ra( x + \varepsilon) \, \diff x \\ \nonumber
        & = 2 \int_\R \varphi_\rb(x) \pr{\ra \in (x - \varepsilon, x + \varepsilon)}\varphi_\ra( x - \varepsilon) \, \diff x \\
        & \ \ \ \ - 2 \int_\R \varphi_\rb(y) \pr{\ra \in (y - \varepsilon, y + \varepsilon)} \varphi_\ra( y - \varepsilon) \, \diff x \label{eq:claimIntMaxTrick2} \\
        & = 0, \nonumber
    \end{align}
    where in \cref{eq:claimIntMaxTrick1} we exploited the symmetry of the integrand functions, \cref{eq:claimIntMaxTrick2} we substituted in the second integral $y = - x$ and used again symmetry.

    \paragraph{Second case: $ z > 0$.} Then
    \begin{align}
        & \ \ \frac{\diff H(z)}{ \diff z} \nonumber \\
        & = 2 \int_\R \varphi_\rb(x) \pr{\ra \in (z - x - \varepsilon, z - x + \varepsilon)} \left(\varphi_\ra( z - x + \varepsilon) - \varphi_\ra( z - x - \varepsilon)\right) \, \diff x \nonumber \\
        & = 2 \int_{-\infty}^{-z} \varphi_\rb(x) \pr{\ra \in (z - x - \varepsilon, z - x + \varepsilon)} \left(\varphi_\ra( z - x + \varepsilon) - \varphi_\ra( z - x - \varepsilon)\right) \, \diff x \nonumber \\
        & \ \ + 2\int_{-z}^{+z} \varphi_\rb(x) \pr{\ra \in (z - x - \varepsilon, z - x + \varepsilon)} \left(\varphi_\ra( z - x + \varepsilon) - \varphi_\ra( z - x - \varepsilon)\right) \, \diff x \nonumber \\
        & \ \ +2\int_{+z}^{+\infty} \varphi_\rb(x) \pr{\ra \in (z - x - \varepsilon, z - x + \varepsilon)} \left(\varphi_\ra( z - x + \varepsilon) - \varphi_\ra( z - x - \varepsilon)\right) \, \diff x \nonumber \\
        & = 2\int_{+z}^{+\infty} \varphi_\rb(x) \pr{\ra \in (z + x - \varepsilon, z + x + \varepsilon)} \left(\varphi_\ra( z + x + \varepsilon) - \varphi_\ra( z + x - \varepsilon)\right) \, \diff x \label{eq:claimIntMaxTrick3} \\
        &  \ \ + 2 \int_{+3z}^{+\infty} \varphi_\rb(x) \pr{\ra \in (z - x - \varepsilon, z - x + \varepsilon)} \left(\varphi_\ra( z - x + \varepsilon) - \varphi_\ra( z - x - \varepsilon)\right) \, \diff x \nonumber \\
        &  \ \ + 2 \int_{+z}^{+3z} \varphi_\rb(x) \pr{\ra \in (z - x - \varepsilon, z - x + \varepsilon)} \left(\varphi_\ra( z - x + \varepsilon) - \varphi_\ra( z - x - \varepsilon)\right) \, \diff x \nonumber \\
        & \ \  +2\int_{- z}^{+ z} \varphi_\rb(x) \pr{\ra \in (z - x - \varepsilon, z - x + \varepsilon)} \left(\varphi_\ra( z - x + \varepsilon) - \varphi_\ra( z - x - \varepsilon)\right) \, \diff x \nonumber \\
        & = 2\int_{+z}^{+\infty} \varphi_\rb(x) \pr{\ra \in (z + x - \varepsilon, z + x + \varepsilon)} \left(\varphi_\ra( z + x + \varepsilon) - \varphi_\ra( z + x - \varepsilon)\right) \, \diff x \nonumber \\
        &  \ \ - 2 \int_{+z}^{+\infty} \varphi_\rb(2z + x) \pr{\ra \in (z + x - \varepsilon, z + x + \varepsilon)} \left(\varphi_\ra( z + x + \varepsilon) - \varphi_\ra( z + x - \varepsilon)\right) \, \diff x \label{eq:claimIntMaxTrick4} \\
        &  \ \ - 2 \int_{-z}^{+z} \varphi_\rb(x - 2z) \pr{\ra \in ( z - x - \varepsilon,  z - x +  \varepsilon)} \left(\varphi_\ra( z - x + \varepsilon) - \varphi_\ra( z - x - \varepsilon)\right) \, \diff x \label{eq:claimIntMaxTrick5} \\
        & \ \  +2\int_{- z}^{+ z} \varphi_\rb(x) \pr{\ra \in (z - x - \varepsilon, z - x + \varepsilon)} \left(\varphi_\ra( z - x + \varepsilon) - \varphi_\ra( z - x - \varepsilon)\right) \, \diff x \nonumber \\
        & = 2\int_{+z}^{+\infty} \left(\varphi_\rb(x) - \varphi_\rb(2z + x) \right) \pr{\ra \in (z + x - \varepsilon, z + x + \varepsilon)}  \left(\varphi_\ra( z + x + \varepsilon) - \varphi_\ra( z + x - \varepsilon)\right) \, \diff x  \label{eq:claimIntMaxTrick6} \\
        &  \ \ + 2 \int_{-z}^{+z} \left( \varphi_\rb(x) - \varphi_\rb(x - 2z)\right) \pr{\ra \in (z - x - \varepsilon, z - x + \varepsilon)}  \left(\varphi_\ra( z - x + \varepsilon) - \varphi_\ra( z - x - \varepsilon)\right) \, \diff x \label{eq:claimIntMaxTrick7},
    \end{align}
    where in \cref{eq:claimIntMaxTrick3} we substituted $x' = -x$ and used the symmetry of the integrand functions, in \cref{eq:claimIntMaxTrick4,eq:claimIntMaxTrick5} we substituted $x' = x - 2z$ and $x' = 2z - x$, respectively,  and used again the symmetry. The expression in \cref{eq:claimIntMaxTrick6} is negative as $\varphi_\rb(x) > \varphi_\rb(2z + x)$ and $ \varphi_\ra(z + x + \varepsilon) < \varphi_\ra(z+x-\varepsilon) $ for $x \ge z$; the expression in \cref{eq:claimIntMaxTrick7} is negative as $\varphi_\rb(x) > \varphi_\rb(x - 2z)$ and $\varphi_\ra(z - x + \varepsilon) < \varphi_\ra(z - x-\varepsilon)  $ for $x \in (-z,z)$.

    \paragraph{Third case: $z < 0$.} This case is similar to the previous one: with the same arguments, we obtain
    \begin{align}
        & \ \ \frac{\diff H(z)}{ \diff z} \nonumber
\\ & = 2\int_{-\infty}^{+z} \left(\varphi_\rb(x) - \varphi_\rb(2z + x) \right) \pr{\ra \in (z + x - \varepsilon, z + x + \varepsilon)}  \left(\varphi_\ra( z + x + \varepsilon) - \varphi_\ra( z + x - \varepsilon)\right) \, \diff x  \label{eq:claimIntMaxTrick10} \\
        &  \ \ + 2 \int_{+z}^{-z} \left( \varphi_\rb(x) - \varphi_\rb(x - 2z)\right) \pr{\ra \in (z - x - \varepsilon, z - x + \varepsilon)}  \left(\varphi_\ra( z - x + \varepsilon) - \varphi_\ra( z - x - \varepsilon)\right) \, \diff x \label{eq:claimIntMaxTrick11}.
    \end{align}
The expression in \cref{eq:claimIntMaxTrick10} is positive as $\varphi_\rb(x) > \varphi_\rb(2z + x)$ and $ \varphi_\ra(z + x + \varepsilon) > \varphi_\ra(z+x-\varepsilon) $ for $x \le z$; the expression in \cref{eq:claimIntMaxTrick11} is positive as $\varphi_\rb(x) > \varphi_\rb(x - 2z)$ and $\varphi_\ra(z - x + \varepsilon) < \varphi_\ra(z - x-\varepsilon)  $ for $x \in (z,-z)$.
\end{proof}

\begin{claim}\label{claim:int-ub-convex}
    For all $x \in \R$, $c \in \left(0, \frac{1}{162}\right)$, and $\varepsilon \in (0, 1)$, it holds that
    \begin{align*}
        \left(\int_{-\varepsilon}^{\varepsilon} e^{-c(x + s)^2} \, \diff s\right)^2
        &\le \left(\int_{-\varepsilon}^{\varepsilon}\frac{e^{-c (x + \varepsilon)^2 } + e^{-c (x - \varepsilon)^2 }}{2}e^{c \varepsilon^2 } \, \diff s\right)^2.
    \end{align*}
\end{claim}
\begin{proof}
Let
\begin{align*}
    f_x(s) = e^{-c(x + s)^2}.
\end{align*}
Since
\begin{align*}
    \int_{-\varepsilon}^{\varepsilon} \frac{e^{-c(x + \varepsilon)^2} + e^{-c(x - \varepsilon)^2}}{2} e^{c\varepsilon^2} \,\diff s
    &= \int_{-\varepsilon}^{\varepsilon} ms + \frac{e^{-c(x + \varepsilon)^2} + e^{-c(x - \varepsilon)^2}}{2} e^{c\varepsilon^2} \,\diff s
\end{align*}
for any $m \in \R$, we choose it to be the angular coefficient of the line passing through $f_x(-\varepsilon)$ and $f_x(\varepsilon)$, and prove the stronger result
\begin{align}
    e^{-c(x+s)^2}
    \le \frac{e^{-c(x + \varepsilon)^2} - e^{-c(x - \varepsilon)^2}}{2\varepsilon} s + \frac{e^{-c(x + \varepsilon)^2} + e^{-c(x - \varepsilon)^2}}{2}e^{c\varepsilon^2} \label{eq:sharpUB:extrema-line}
\end{align}
for all $s \in (-\varepsilon, \varepsilon)$.
In fact, the right hand side of \cref{eq:sharpUB:extrema-line} is the equation for the line passing by the extrema of $f_x$ in $(-\varepsilon, \varepsilon)$ lifted by a factor of $e^{c\varepsilon^2}$.
Therefore, the results holds trivially if $f_x$ is convex in the entire range $(-\varepsilon, \varepsilon)$, which is true when $\abs{x} > 1 + \frac{1}{\sqrt{2c}}$.
Moreover, the factor $e^{c\varepsilon^2}$ ensures the result for $x = 0$, so, we follow with the analysis of the case $x \in \left(0, 1 + \frac{1}{\sqrt{2c}}\right]$ and the remaining case $x \in \left[-1 - \frac{1}{\sqrt{2c}}, 0\right)$ follows by symmetry.

Dividing both side of \cref{eq:sharpUB:extrema-line} by $e^{-c(x + s)^2}$, we obtain
\begin{align}\label{eq:sharpUB:hyperbolic}
    1
    &\le e^{2csx + cs^2} \left[\frac{e^{-c\varepsilon^2}s}{\varepsilon} \cdot \frac{e^{-2c\varepsilon x} - e^{2c\varepsilon x}}{2} + \frac{e^{-2c\varepsilon x} + e^{2c\varepsilon x}}{2}\right]
    \\&= e^{2csx + cs^2} \left[-\frac{e^{-c\varepsilon^2}s}{\varepsilon} \sinh(2c\varepsilon x) + \cosh(2c\varepsilon x)\right]. \nonumber
\end{align}
Let $g(x)$ be the right hand side of this inequality.
Then
\begin{align*}
    g'(x) &= 2csg(x) + 2c\varepsilon e^{2csx + cs^2} \left[-\frac{e^{-c\varepsilon^2}s}{\varepsilon} \cosh(2c\varepsilon x) + \sinh(2c\varepsilon x)\right]
    \\&= 2ce^{2csx + cs^2} \left[\cosh(2c\varepsilon x) \left(s - se^{-c\varepsilon^2}\right) + \sinh(2c\varepsilon x) \left(\varepsilon - \frac{s^2}{\varepsilon} e^{-c\varepsilon^2}\right)\right].
\end{align*}
If $s \in [0, \varepsilon)$, then $s \ge se^{-c\varepsilon^2}$ and $\varepsilon \ge \frac{\varepsilon^2}{\varepsilon} e^{-c\varepsilon^2} \ge \frac{s^2}{\varepsilon}e^{-c\varepsilon^2}$, hence $g'(x) \ge 0$. Since $g(0) \ge 1$, this ensures \cref{eq:sharpUB:hyperbolic}.

The sub-case $s \in (-\varepsilon, 0)$ offers much more resistance.
To analyse it we exploit that $x \in \left(0, 1 + \frac{1}{\sqrt{2c}}\right)$ implies that $cx \le \sqrt{2c}$ for $c < \frac{1}{2}$ and make extensive use of Taylor's theorem to approximate the exponential functions.

We start by rewriting \cref{eq:sharpUB:hyperbolic} as
\begin{align}\label{eq:sharpUB:negatives}
    \varepsilon e^{-2csx -cs^2} \le e^{2c\varepsilon x} \left(\frac{\varepsilon}{2} - \frac{s}{2}e^{-c\varepsilon^2}\right) + e^{-2c\varepsilon x} \left(\frac{\varepsilon}{2} + \frac{s}{2}e^{-c\varepsilon^2}\right).
\end{align}
By Taylor's theorem, there exist $\lambda_1, \lambda_2 \in [0,2c\varepsilon x] \subseteq [0, 2\sqrt{2c} \varepsilon]$, $\lambda_3 \in [0,-2csx] \subseteq [0, 2\sqrt{2c} \varepsilon]$, $\lambda_4 \in [0,cs^2]$, $\lambda_5 \in [0,c\varepsilon^2]$ such that
\begin{align}
    & e^{+2c\varepsilon x} = 1 + 2c\varepsilon x + 2c^2 \varepsilon^2 x^2 + \frac{4}{3}c^3\varepsilon^3 x^3 e^{\lambda_1}, \nonumber \\ & e^{-2c\varepsilon x} = 1 - 2c\varepsilon x + 2c^2 \varepsilon^2 x^2 - \frac{4}{3}c^3\varepsilon^3 x^3 e^{\lambda_2}, \nonumber \\ & e^{-2csx} = 1 - 2cs x + 2c^2 s^2 x^2 - \frac{4}{3}c^3s^3 x^3 e^{\lambda_3}, \nonumber \\ & e^{-cs^2} = 1 -cs^2 + \frac{c^2s^4}{2} e^{-\lambda_4}, \label{eq:sharpUB:Taylor4}\\
    & e^{-c\varepsilon^2} = 1 -c\varepsilon^2 + \frac{c^2\varepsilon^4}{2} e^{-\lambda_5} \label{eq:sharpUB:Taylor5},
\end{align}
where we used second order approximations for the first three terms and first order approximations for the last two. Plugging those in \cref{eq:sharpUB:negatives} we obtain
\begin{align*}
    \varepsilon e^{-cs^2} \bigl(1 - 2cs x + 2c^2 s^2 x^2 - \frac{4}{3}c^3s^3 x^3 e^{\lambda_3}\bigr)
    &\le \bigl(1 + 2c\varepsilon x + 2c^2 \varepsilon^2 x^2 + \frac{4}{3}c^3\varepsilon^3 x^3 e^{\lambda_1}\bigr) \bigl(\frac{\varepsilon}{2} - \frac{s}{2}e^{-c\varepsilon^2}\bigr)
    \\&\quad+ \bigl(1 - 2c\varepsilon x + 2c^2 \varepsilon^2 x^2 - \frac{4}{3}c^3\varepsilon^3 x^3 e^{\lambda_2}\bigr) \bigl(\frac{\varepsilon}{2} + \frac{s}{2}e^{-c\varepsilon^2}\bigr).
\end{align*}
The latter becomes
\begin{align*}
    & \varepsilon\left(1 - e^{-cs^2}\right) + 2cs\varepsilon x \left(e^{-cs^2} - e^{-c\varepsilon^2}\right) + 2c^2\varepsilon x^2\left(\varepsilon^2 - s^2 e^{-cs^2}\right) \\
    & + \frac{4}{3}c^3 \varepsilon x^3 \left(\varepsilon^2 e^{\lambda_1}\left(\frac{\varepsilon}{2} - \frac{s}{2}e^{-c\varepsilon^2}\right) - \varepsilon^2 e^{-\lambda_2}\left(\frac{\varepsilon}{2} + \frac{s}{2}e^{-c\varepsilon^2}\right) + s^3 e^{\lambda_3 - cs^2}\right) \ge 0
\end{align*}
Now, notice that
\begin{align*}
    \varepsilon^2 e^{\lambda_1}\left(\frac{\varepsilon}{2} - \frac{s}{2}e^{-c\varepsilon^2}\right) - \varepsilon^2 e^{-\lambda_2}\left(\frac{\varepsilon}{2} + \frac{s}{2}e^{-c\varepsilon^2}\right) \ge 0,
\end{align*}
as
$ \frac{\varepsilon}{2} - \frac{s}{2}e^{-c\varepsilon^2} \ge \frac{\varepsilon}{2} + \frac{s}{2}e^{-c\varepsilon^2}$ since $-\varepsilon \le s < 0$, and $ \varepsilon^2 e^{\lambda_1} \ge \varepsilon^2 \ge \varepsilon^2 e^{-\lambda_2}$.
Furthermore, observe  that $ s^3 e^{\lambda_3 - cs^2} \ge 2s^3 $ as $s < 0$ and $\lambda_3 \le 2\sqrt{2c} \varepsilon \le \frac{1}{2}$ if $c \le \frac{1}{32}$.
Thus, the inequality is true if
\begin{align*}
    & \varepsilon\left(1 - e^{-cs^2}\right) + 2cs\varepsilon x \left(e^{-cs^2} - e^{-c\varepsilon^2}\right) + 2c^2\varepsilon x^2\left(\varepsilon^2 - s^2 e^{-cs^2}\right) + \frac{8}{3}c^3 s^3 \varepsilon x^3  \ge 0.
\end{align*}
Applying \cref{eq:sharpUB:Taylor4,eq:sharpUB:Taylor5}, the latter inequality yields that
\begin{align*}
    &\varepsilon \left(cs^2 - \frac{c^2s^4}{2}e^{-\lambda_4}\right) + 2cs\varepsilon x \left(c \varepsilon^2 - cs^2 - \frac{c^2 \varepsilon^4}{2}e^{-\lambda_5} + \frac{c^2 s^4}{2}e^{-\lambda_4} \right)
    \\&\qquad\qquad+ 2c^2 \varepsilon x^2 \left(\varepsilon^2 - s^2 + cs^4 - \frac{c^2s^6}{2} e^{-\lambda_4}\right) + \frac{8}{3}c^3 s^3 \varepsilon x^3
    \\&\qquad= \varepsilon c s^2 - \frac{c^2 s^4 \varepsilon}{2}e^{-\lambda_4} - c^3s\varepsilon^5 x e^{-\lambda_5} + c^3 s^5 \varepsilon x e^{-\lambda_4} + \left(2c^3 s^4 \varepsilon x^2 - c^4 s^6 \varepsilon x^2 e^{-\lambda_4}\right) + \frac{8}{3}c^3 s^3 \varepsilon x^3
    \\&\qquad\qquad+ 2c^2 \varepsilon x\left(\varepsilon^2 - s^2\right) \left(x + s\right).
\end{align*}
Now observe that
\begin{align*}
    \left(2c^3 s^4 \varepsilon x^2 - c^4 s^6 \varepsilon x^2 e^{-\lambda_4}\right) \ge 0
\end{align*}
as $c < 1, s \le \varepsilon \le 1, e^{-\lambda_4} < 1$;
$- c^3s\varepsilon^5 x e^{-\lambda_5} > 0$ as $ s < 0$;
\begin{align}
    \varepsilon c s^2 - \frac{c^2 s^4 \varepsilon}{2}e^{-\lambda_4} + c^3 s^5 \varepsilon x e^{-\lambda_4} + \frac{8}{3}c^3 s^3 \varepsilon x^3
    &\ge c s^2 \varepsilon - \frac{c^2 s^2 \varepsilon^3}{2} - c^2\sqrt{2c} s^2 \varepsilon^4 - \frac{8}{3}c^3 s^2 \varepsilon^2 x^3 \nonumber
    \\&> c s^2 \varepsilon - \frac{c^2 s^2 \varepsilon^3}{2} - c^2\sqrt{2c} s^2 \varepsilon^4 - 6c\sqrt{2c}s^2 \varepsilon^2 \label{eq:sharpUB:trick3}
    \\&= cs^2 \varepsilon\left(1 - \frac{c \varepsilon^2}{2} - c^2\sqrt{2c}\varepsilon^3 - 6\sqrt{2c} \varepsilon
    \right) \nonumber
    \\&\ge cs^2 \varepsilon\left(1 - \frac{c }{2} - c^2\sqrt{2c} - 6\sqrt{2c}
    \right) \label{eq:sharpUB:trick4}
    \\&\ge \frac{cs^2 \varepsilon}{5} \label{eq:sharpUB:trick5},
\end{align}
where in
\cref{eq:sharpUB:trick3} we used that $cx \le \sqrt{2c}$, in \cref{eq:sharpUB:trick4} that $\varepsilon \le 1$, and in \cref{eq:sharpUB:trick5}
we used i) $c^2\sqrt{2c}\le \frac{c}{2}$ when $c\le \frac{1}{\sqrt[3]{2}}$, ii) $c < \sqrt{2c}$ since $c < 1$ and iii) $1 - 7\sqrt{2c}\ge \frac{1}{5}$, whenever $c \le \frac{1}{162}$.

Going back to the inequality, we now have that
\begin{align*}
    \frac{cs^2 \varepsilon}{5} + 2c^2 \varepsilon x\left(\varepsilon^2 - s^2\right) \left(x + s\right).
\end{align*}
If $x \ge \abs{s}$ the latter is positive, otherwise it becomes
\begin{align*}
    \frac{cs^2 \varepsilon}{5} + 2c^2 \varepsilon x\left(\varepsilon^2 - s^2\right) \left(x + s\right)
    &\ge \frac{cs^2 \varepsilon}{5} + 2c^2 \varepsilon x^2\left(\varepsilon^2 - s^2\right) - 2c^2 \varepsilon s^2\left(\varepsilon^2 - s^2\right)
    \\&\ge \frac{cs^2 \varepsilon}{5} - 2c^2 \varepsilon s^2 + 2c^2 \varepsilon x^2\left(\varepsilon^2 - s^2\right)
    \\&\ge cs^2\varepsilon \left(\frac{1}{5} - 2c\right),
\end{align*}
which is positive for $c < \frac{1}{10}$.
\end{proof}

\begin{claim}\label{claim:lb-exp-int}
For all $x \in \R$, $c \in \left(0, \frac{1}{10}\right)$, and $\varepsilon \in (0, 1)$, it holds that
\begin{align}\label{eq:expTrick}
    \left(\int_{x - \varepsilon}^{x + \varepsilon} \exp \left(- c y^2\right) \, \diff y\right)^2
    &\ge \int_{x - \varepsilon}^{x + \varepsilon} \exp\left(-c(x-\varepsilon)^2\right) \,\diff y \cdot \int_{x - \varepsilon}^{x + \varepsilon} \exp\left(-c(x+\varepsilon)^2\right) \,\diff y.
\end{align}
\end{claim}
\begin{proof}
We can express \cref{eq:expTrick} as
\begin{align*}
    &\left[\int_{x-\varepsilon}^{x+\varepsilon} \exp\left(-cy^2\right) \, \diff y\right]^2 - \left[\int_{x-\varepsilon}^{x+\varepsilon} \exp \left(-c(x^2 + \varepsilon^2)\right) \, \diff y\right]^2
    \\&\quad= \left[\int_{x-\varepsilon}^{x+\varepsilon} \exp\left(-cy^2\right) - \exp\left(-c(x^2 + \varepsilon^2)\right) \, \diff y\right] \cdot \left[\int_{x-\varepsilon}^{x+\varepsilon}\exp\left(-cy^2\right) + \exp\left(-c(x^2 + \varepsilon^2)\right) \, \diff y \right]
    \\&\quad\ge 0,
\end{align*}
which holds if and only if
\begin{align}
    \int_{-\varepsilon}^{+\varepsilon} \exp\left(-c(x+s)^2\right) \, \diff s
    &\ge \int_{-\varepsilon}^{+\varepsilon} \exp\left(-c(x^2+\varepsilon^2)\right) \, \diff s. \label{eq:lb-exp-int}
\end{align}
The result is immediate for $x = 0$, so we assume $x > 0$ and the claim follows by symmetry.
Let
\begin{align*}
    f_x(s) = \exp(-c(x+s)^2).
\end{align*}
We provide distinct arguments depending on whether $x$ is small or large.

\paragraph{Case $x \in (0, 1)$.}
Since we assume $c < \frac{1}{8}$ and $\varepsilon < 1$, we have for any $x \le 1$ that $f_x$ is concave in $(-\varepsilon, \varepsilon)$.
That is,
\begin{align*}
    f_x(s)
    &\ge \frac{f_x(\varepsilon) - f_x(-\varepsilon)}{2\varepsilon} s + \frac{f_x(\varepsilon) + f_x(\varepsilon)}{2},
\end{align*}
for all $s \in (-\varepsilon, \varepsilon)$.
Thus,
\begin{align*}
    \int_{-\varepsilon}^{\varepsilon} f_x(s) \,\diff s
    &\ge \int_{-\varepsilon}^{\varepsilon} \frac{f_x(\varepsilon) - f_x(-\varepsilon)}{2\varepsilon} s + \frac{f_x(\varepsilon) + f_x(-\varepsilon)}{2} \,\diff s
    \\&= \int_{-\varepsilon}^{\varepsilon} \frac{f_x(\varepsilon) + f_x(-\varepsilon)}{2} \,\diff s
    \\&= \int_{-\varepsilon}^{\varepsilon} \exp\left(-c(x^2+\varepsilon^2)\right) \cdot \frac{\exp\left(-2cx\varepsilon\right) + \exp\left(2cx\varepsilon\right)}{2} \,\diff s
    \\&\ge \int_{-\varepsilon}^{\varepsilon} \exp\left(-c(x^2+\varepsilon^2)\right) \,\diff s.
\end{align*}

\paragraph{Case $x \ge 1$.}
The integral on the right hand side of \cref{eq:lb-exp-int} has the same value for any affine integrand $r_x$ for which $r_x(0) = \exp\left(-c(x^2+\varepsilon^2)\right)$.
Thus, proving that
$f_x(s) \ge r_x(s)$,
for all $s \in (-\varepsilon, \varepsilon)$, concludes the proof.

In particular, we can choose
\begin{align*}
    r_x(s) = f_x'(0) \cdot s + \exp\left(-c(x^2+\varepsilon^2)\right).
\end{align*}
Since
\begin{align*}
    f_x'(s) = -2c(x+s)\exp\left(-c(x+s)^2\right),
\end{align*}
we aim to show that
\begin{align*}
    \exp\left(-c(x+s)^2\right)
    &\ge -2csx\exp\left(-cx^2\right) + \exp\left(-c(x^2+\varepsilon^2)\right)
\end{align*}
for $s \in (-\varepsilon, \varepsilon)$.
Dividing by $\exp(-c(x^2 + s^2))$ and rearranging, we obtain
\begin{align}\label{eq:mainForShowing}
    \exp(-2csx) + 2csx \exp\left(cs^2\right) - \exp \left(-c\left(\varepsilon^2 - s^2\right)\right) \ge 0.
\end{align}
Now, if $s \ge 0$, we have that
\begin{align}
    \exp(-2csx) + 2csx \exp(cs^2) - \exp (-c(\varepsilon^2 - s^2))
    &\ge 1 - 2csx +  2csx (1 + cs^2) - \exp(-c \varepsilon^2) \label{eq:toshowtrick}
    \\&= 1 + 2c^2s^3x - \exp(-c \varepsilon^2) \nonumber
    \\&\ge 2c^2s^3x \nonumber
    \\&\ge 0, \nonumber
\end{align}
where in \cref{eq:toshowtrick} we used that $e^y \ge 1 + y$.

Now consider the sub-case $s < 0$.
By Taylor's theorem,
\begin{align*}
    \exp(y) &= 1 + y + \frac{y^2}{2} + \frac{\exp(\xi_1) \cdot y^3}{6} \end{align*}
and
\begin{align*}
    \exp(y) & = 1 + y + \frac{\exp(\xi_2) \cdot y^2}{2}, \end{align*}
for some $\xi_1,\xi_2 \in [0,y]$.
Letting $\ell = -s \in (0, 1)$, we have
\begin{align*}
    \exp(2c\ell x) \ge 1 + 2c\ell x + \frac{(2c\ell x)^2}{2} + \frac{(2c\ell x)^3}{6}
\end{align*}
and
\begin{align*}
    \exp(c\ell^2) &\le 1 + c\ell^2 + \frac{\exp(c\ell^2) (c\ell^2)^2}{2}
    \\&\le 1 + c\ell^2 + \frac{\left(1+3(c\ell^2)\right) (c\ell^2)^2}{2}.
\end{align*}
since $e^y \le 1 + 3y$ for $0 \le y \le 1$.
Finally, applying this to \cref{eq:mainForShowing}, we have
\begin{align*}
    & \exp(-2csx) + 2csx \exp\left(cs^2\right) - \exp \left(-c\left(\varepsilon^2 - s^2\right)\right)
    \\&\qquad\ge \exp(2c\ell x) - 2c\ell x \exp\left(c\ell^2\right) - 1
    \\&\qquad\ge 1 + 2c\ell x + \frac{(2c\ell x)^2}{2} + \frac{(2c\ell x)^3}{6} - 2c\ell x \left(1 + c\ell^2 + \frac{c^2\ell^4(1 + 3c\ell^2)}{2}  \right) - 1
    \\&\qquad= 2c^2\ell^2x^2 + \frac{4}{3}c^3\ell^3x^3 - 2c\ell x \left( c\ell^2 + \frac{c^2\ell^4(1 + 3c\ell^2)}{2} \right)
    \\&\qquad= 2c^2\ell^2x (x - \ell) + c^3 \ell^3 x\left( \frac{4}{3}x^2 - \ell^2(1 + 3c\ell^2)\right).
\end{align*}
The latter is non negative for $x \ge 1$ and $c \le \frac{1}{9}$, since $\ell = -s\le\varepsilon < 1$, so that $\frac{4}{3}x^2 - \ell^2(1 + 3c\ell^2)\ge \frac{4}{3} - 1 - \frac 13 = 0$.
\end{proof}

\section{Proofs omitted}
\subsection{Proof of Lemma \ref{lem:n-subsets}}
If $\alpha \ge \frac{1}{2}$, then $2\alpha^2 n \ge \alpha n$, and the result holds trivially. So we assume $\alpha < \frac{1}{2}$.

Let $k$ be any integer, and let $\CC = \{C_1, C_2, \ldots, C_k\}$, with each $C_i$ drawn uniformly from the collection of subsets of $[n]$ with size $\alpha n$.
Given $i, j \in [k]$, if $i \neq j$, then $ \abs{C_i \cap C_j} $ follows a hypergeometric distribution $ \gH(n, \alpha n, \alpha n) $, which has mean $\alpha^2 n$.
As argued in \cite[Theorem 1.17]{doerr2011hypergeometric}, the known Chernoff bounds (\cref{lemma:chernoff-hoeffding}) hold for hypergeometric distributions. Hence,
\begin{align*}
    \pr{\lvert C_i \cap C_j \rvert > 2\alpha^2 n}
&\le \exp\left( -\frac{\alpha^2 n}{3} \right).
\end{align*}
Finally, for the event of interest, we have that
\begin{align}
    \pr{\bigcap_{i \neq j \in [k]} \left\{\lvert C_i \cap C_j \rvert  \le 2\alpha^2 n \right\}} \nonumber
    &= 1 - \pr{\bigcup_{i \neq j \in [k]} \left\{\lvert C_i \cap C_j \rvert > 2\alpha^2 n \right\}}  \nonumber
    \\&\ge 1 - \binom{k}{2}\exp\left( -\frac{\alpha^2 n}{3} \right) \nonumber
\\&\ge 1 - 2^{\frac{\alpha^2 t}{3}} \cdot \exp\left( -\frac{\alpha^2 n}{3} \right) \label{eq:setting_k}
\\&> 0, \nonumber
\end{align}
where in \cref{eq:setting_k} we have chosen $k = 2^{\frac{\alpha^2 n}{6}}$.

\subsection{Proof of Lemma \ref{lem:uniform-approximation}}
By the distribution of $\rvx$,
\begin{align*}
    \pr{\rvx \in \dBall(\vz,\varepsilon)}
    &= \int_{\dBall(\vz,\varepsilon)} \frac{1}{\left(2\pi\sigma^2\right)^{\frac{d}{2}}} \cdot \exp \left(-\frac{\twoNorm{\vx}^2}{2\sigma^2}\right) \diff \vx.
\end{align*}
Since $\dBall(\vz,\varepsilon) \subseteq \dBall(\vzero,2)$ and
for all $\vx \in \R^d$ it holds that $\lVert\vx\rVert_2 \le \sqrt{d} \cdot \lVert\vx\rVert_\infty$, and, thus,
\begin{align*}
    \exp\left(-\frac{2d}{\sigma^2}\right)
    \le \exp\left(-\frac{\twoNorm{\vx}^2}{2\sigma^2}\right)
    \le 1.
\end{align*}
The thesis follows by noting that the hypercube $\dBall(\vz,\varepsilon)$ has measure $(2\varepsilon)^d$.

\subsection{Proof of Lemma \ref{lem:tightness:covariance}}
Inheriting the setup from the proof of \cref{lem:variance} and proceeding analogously we obtain that $\sigma_{\ra}^2 = \alpha n\left(1 - \frac{\alpha}{2}\right)$ and $\sigma_\rb^2 = \frac{\alpha^2 n}{2}$.
We diverge from that argument after \cref{eq:last-equality}.
Preserving equality for a bit longer, we have that
\begin{align*}
    \left(\pr{\ry_S = 1, \ry_T = 1}\right)^{\frac{1}{d}}
    &= \int_{\R} \varphi_\rb(x) \cdot \left(\pr{\ra \in (z-x-\varepsilon, z-x+\varepsilon)}\right)^2 \,\diff x
    \\&= \int_{\R} \varphi_\rb(x) \cdot \left(\int_{z-x-\varepsilon}^{z-x + \varepsilon} \varphi_\ra(y) \,\diff y\right)^2 \,\diff x.
\end{align*}

The hypothesis on $n$ implies that $2\sigma_a^2 \ge 10$, so, by \cref{claim:lb-exp-int},
\begin{align*}
    \left(\int_{z-x-\varepsilon}^{z-x+\varepsilon} \varphi_\ra(y) \,\diff y\right)^2
    &\ge (2\varepsilon)^2 \cdot \varphi_\ra(z-x-\varepsilon) \cdot \varphi_\ra(z-x+\varepsilon)
    \\&= \frac{(2\varepsilon)^2}{2 \pi \sigma_\ra^2} \cdot \exp\left(-\frac{(z-x-\varepsilon)^2}{2\sigma_\ra^2}\right) \cdot \exp\left(-\frac{(z-x+\varepsilon)^2}{2\sigma_\ra^2}\right)
    \\&= e^{-\varepsilon^2/\sigma_\ra^2} \cdot \frac{1}{\sqrt{2}} \cdot \frac{(2\varepsilon)^2}{\sqrt{2 \pi \sigma_\ra^2}} \cdot \frac{1}{\sqrt{\pi \sigma_\ra^2}} \cdot \exp\left(-\frac{(z-x)^2}{\sigma_\ra^2}\right)
    \\&= e^{-\varepsilon^2/\sigma_\ra^2} \cdot \frac{1}{\sqrt{2}} \cdot \frac{(2\varepsilon)^2}{\sqrt{2 \pi \sigma_\ra^2}} \cdot \varphi_{\ra/\sqrt{2}}(z-x).
\end{align*}
Then, as before, we can reduce the main integral to a convolution.
Namely, it holds that
\begin{align*}
    \int_{\R} \varphi_\rb(x) \cdot \varphi_{\ra/\sqrt{2}}(z-x) \,\diff x
    &= \varphi_{\rb+\ra/\sqrt{2}}(z)
    \\&= \frac{1}{\sqrt{2\pi\sigma_{\rb+\ra/\sqrt{2}}^2}} \cdot \exp\left(-\frac{z^2}{2\sigma_{\rb+\ra/\sqrt{2}}^2}\right).
\end{align*}

Altogether, we have that
\begin{align*}
    \left(\pr{\ry_S = 1, \ry_T = 1}\right)^{\frac{1}{d}}
    &\ge \frac{(2\varepsilon)^2}{2 \pi} \cdot \frac{1}{\sqrt{2\sigma_\ra^2\sigma_{\rb+\ra/\sqrt{2}}^2}} \cdot \exp\left(-\frac{\varepsilon^2}{\sigma_\ra^2} - \frac{z^2}{2\sigma_{\rb+\ra/\sqrt{2}}^2}\right)
    \\&= \frac{(2\varepsilon)^2}{2 \pi \alpha n} \cdot \frac{1}{\sqrt{1 - \frac{\alpha^2}{4}}} \cdot \exp\left(-\frac{1}{\alpha n} \cdot \left(\frac{2\varepsilon^2}{2 - \alpha} + \frac{2z^2}{2 + \alpha}\right)\right).
\end{align*}
where the last equality follows from recalling that $\sigma_\rb^2 = \frac{\alpha^2 n}{2}$ and $\sigma_{\ra}^2 = \alpha n\left(1 - \frac{\alpha}{2}\right)$, which implies that $\sigma_{\rb+\ra/\sqrt{2}}^2 = \frac{\alpha^2 n}{2} + \frac{\alpha n}{2}\left(1 - \frac{\alpha}{2}\right)$.
Finally, the hypotheses $z \in [-1, 1]$, $\varepsilon \in (0, 1)$, and $\alpha \in \left(0, \frac{1}{2} \right)$ imply that $\frac{2\varepsilon^2}{2 - \alpha} + \frac{2z^2}{2 + \alpha} < 3$.

\subsection{Proof of Theorem \ref{thm:target}} Let $n = k \cdot \frac{144d}{\alpha^2}\left(\log\frac{1}{\varepsilon} + \log d + \log\frac{1}{\alpha}\right)$ with $k \in \N$.
    By \cref{lem:chebyshev}, for any $\vz \in [-1,1]^d$, the probability than no subset-sum is sufficiently close to $\vz$ is at most $\left(\frac{2}{3}\right)^k$.
    Leveraging the fact that it is possible to cover $[-1, 1]^d$ by $\frac{1}{\varepsilon^d}$ hypercubes of radius $\varepsilon$, we can ensure that the probability of failing to approximate any $\vz$, up to error $2\varepsilon$, is, by the union bound, at most
    \begin{align*}
        \frac{1}{\varepsilon^d}\cdot \left(\frac{2}{3}\right)^k & = 2 ^{- k \log \frac{3}{2} + d \log \frac{1}{\varepsilon}} \\
        & =\exp \left[{ - \ln 2\cdot  \frac{n - \frac{144 d^2}{\alpha^2 \log \frac{3}{2}} \log \frac{1}{\varepsilon} \left(\log \frac{1}{\varepsilon} + \log d + \log \frac{1}{\alpha} \right)}{\frac{144 d}{\alpha^2 \log \frac{3}{2}}\left(\log \frac{1}{\varepsilon} + \log d + \log \frac{1}{\alpha}\right)} }\right].
    \end{align*}
Thus, we can conclude the result for
    \begin{align*}
        n \ge \frac{144}{\log \frac{3}{2}} \cdot \frac{d^2}{\alpha^2} \log\frac{1}{\varepsilon} \cdot \left(\log\frac{1}{\varepsilon} + \log d + \log\frac{1}{\alpha}\right).
    \end{align*}
 \section{Generalisation of our result}\label{sec:generalisation}

If the target value $\vz$ lies in the hypercube $[-\lambda \sqrt{n},\lambda \sqrt{n}]^d$, for some $\lambda > \frac{1}{\sqrt{n}}$, we have slightly different bounds for the expectation and for the variance of $\ry$.
In particular, \cref{lem:expectation} would give
\begin{align}\label{eq:expect:general}
    e^{-\frac{2\lambda^2 d}{\alpha }} \frac{(2\varepsilon)^d\abs{\CC}}{\left(2 \pi \alpha n\right)^{\frac{d}{2}}} \le \expect{\ry} \le \frac{(2\varepsilon)^d\abs{\CC}}{\left(2 \pi \alpha n\right)^{\frac{d}{2}}}.
\end{align}

On the other hand, as the proof of \cref{lem:variance} never uses that $\vz \in [-1,1]^d$ but only exploits the bound on the expectation, it would yield
\begin{align}\label{eq:variance:general}
    \var{\ry} \le \frac{(2\varepsilon)^{2d} \abs{\CC}^2}{(2\pi \alpha n ) ^d} \left[(1 - 4\alpha^2)^{-\frac{d}{2} - e^{-\frac{4\lambda^2 d}{\alpha }}}\right] + \frac{(2\varepsilon)^d \abs{\CC}}{(2\pi \alpha n)^{\frac{d}{2}}}.
\end{align}

We focus on the case $\lambda = \frac{1}{2}\sqrt{\frac{\alpha }{17 d}}$ when $n > \frac{68 d}{\alpha}$ (which implies $\lambda \sqrt{n} > 1$).
Thus, we have a new estimation for the probability to hit a single value.

\begin{lemma}\label{lem:chebyshev:generalised}
    Given $d, n \in \N$,
    $\varepsilon \in (0, 1)$, and
    $\alpha \in (0, \frac{1}{6}]$,
    let $\rvx_1, \ldots, \rvx_n$ i.d.d. following $\Normal{\vzero, \mI_d}$,
    $\vz \in [-\lambda \sqrt{n}, \lambda \sqrt{n}]^d$, with $\lambda = \frac{1}{2}\sqrt{\frac{\alpha}{17 d} }$, and
    $\CC \subseteq \binom{[n]}{\alpha n}$.
If any two subsets in $\CC$ intersect in at most $2\alpha^2n$ elements,
    $\alpha \le \frac{1}{6\sqrt{d}}$, and
    $$n \ge \frac{144d}{\alpha^2}\left(\log\frac{1}{\varepsilon} + \log d + \log\frac{1}{\alpha}\right),$$
    then
    \begin{align*}
        \pr{\ry \ge 1} \ge \frac{1}{3}.
    \end{align*}
\end{lemma}
\begin{proof}
By Chebyshev's inequality, it holds that
\begin{align*}
    \pr{\ry \ge 1} & \ge \pr{\abs{\ry - \expect{\ry}} < \frac{\expect{\ry}}{2}} \\
    & \ge 1 - \frac{4 \cdot \var{\ry}}{\expect{\ry}^2}.
\end{align*}
Notice that $\frac{4\lambda^2 d}{\alpha} = \frac{1}{17}$. Hence, using \cref{eq:expect:general,eq:variance:general}, we get that
\begin{align*}
    \frac{4 \cdot \var{\ry}}{\expect{\ry}^2}
    &\le 4 \cdot \frac{e^{\frac{1}{17}} \cdot \left(2 \pi \alpha n\right)^{d}}{(2\varepsilon)^{2d} \abs{\CC}^2} \cdot \left[\frac{(2\varepsilon)^{2d} \abs{\CC}^2}{(2 \pi \alpha n)^d} \cdot \left[(1 - 4\alpha^2)^{-\frac{d}{2}} - e^{-\frac{1}{17}}\right] + \frac{(2\varepsilon)^d \abs{\CC}}{(2 \pi \alpha n)^{\frac{d}{2}}}\right]
    \\&= 4 \cdot \left(\frac{e^{\frac{1}{17}}}{(1 - 4\alpha^2)^{\frac{d}{2}}} - 1\right) + \frac{4e^{\frac{1}{17}} \cdot (2 \pi \alpha n)^{\frac{d}{2}}}{(2\varepsilon)^d \abs{\CC}}.
\end{align*}
Note that
\cref{claim:ub-cov-term} holds exactly as it is for the ratio
\begin{align*}
    \frac{e^{\frac{1}{17}}}{(1 - 4\alpha^2)^{\frac{d}{2}}}
\end{align*}
obtaining the same bound for $n \ge \frac{68d}{\alpha}$ and $\alpha \le \frac{1}{6\sqrt{d}}$, which yields
\begin{align*}
    4 \cdot \left(\frac{e^{\frac{1}{17}}}{(1 - 4\alpha^2)^{\frac{d}{2}}} - 1\right)
    \le \frac{1}{2}.
\end{align*}
Furthermore, also \cref{claim:ub-var-term} is true replacing $e^{\frac{4d}{\alpha n}}$ by $ e^{\frac{1}{17}}$. Thus, as $n \ge \frac{144d}{\alpha^2}\left(\log\frac{1}{\varepsilon} + \log d + \log\frac{1}{\alpha}\right)$ and $\alpha \le \frac{1}{6}$, \cref{claim:ub-var-term} implies that
\begin{align*}
    \frac{4e^{\frac{1}{17}} \cdot (2 \pi \alpha n)^{\frac{d}{2}}}{(2\varepsilon)^d \abs{\CC}}
    \le \varepsilon.
\end{align*}
\end{proof}

We remark that we cannot let $\lambda$ be asymptotically greater than $\sqrt{\frac{\alpha}{d}}$ otherwise our method fails. Indeed, by \cref{remark:tightness:generalised}, the term $ \frac{4\var{\ry}}{\expect{\ry}^2} $ is at least
\begin{align*}
    4 \cdot \left(\frac{e^{\frac{4 \lambda^2 d}{\alpha} - \frac{3 \lambda^2 d}{\alpha}}}{\left(1 - \frac{\alpha^2}{4}\right)^{\frac{d}{2}}} - 1\right).
\end{align*}
The latter
is greater than or equal to $1$ if $ \lambda \ge \sqrt{\frac{\alpha}{d}} $ since $ e^{\frac{\lambda^2 d}{\alpha}} \ge 1 + \frac{ \lambda^2 d}{\alpha}$.

We are ready to state our first generalised version of \cref{thm:target}.
\begin{theorem}\label{thm:generalisation1}
    For given $d$ and $\varepsilon \in (0, 1)$, let $\rvx_1, \dots, \rvx_n$ be $n$ independent standard normal $d$-dimensional random vectors and let $\alpha \in (0, \frac{1}{6\sqrt{d}}]$.
    There exist two universal constants $C > \delta > 0$ such that, if
    $$
        n \ge C \frac{d^2}{\alpha^2} \left(\log\frac{1}{\varepsilon} + \log d + \log\frac{1}{\alpha}\right)^2,
    $$
    the following holds with probability at least
    \begin{align*}
        1 - \exp \left[{- \ln 2 \cdot \left(\frac{n }{\delta \frac{d}{\alpha^2}\left(\log \frac{1}{\varepsilon} + \log d + \log \frac{1}{\alpha}\right) } - d \log \frac{1}{\varepsilon}\right)}\right]:
    \end{align*}
    for all $\vz \in \left[-\lambda \sqrt{n},\lambda \sqrt{n}\right]^d$, with $\lambda = \frac{1}{2}\sqrt{\frac{\alpha}{17d}}$, there exists a subset $S_{\vz} \subseteq [n]$, such that
    \begin{equation*}
        \bigl\lVert \vz - \sum_{i \in S_\vz} \rvx_i \bigr\rVert_\infty \le 2\varepsilon.
    \end{equation*}
	Moreover, the property above remains true even if we restrict to subsets of size $\alpha n$.
\end{theorem}
\begin{proof}
Let $ \frac {n}{  \frac{144d}{\alpha^2}\left(\log\frac{1}{\varepsilon} + \log d + \log\frac{1}{\alpha}\right)} = k \ge 1$ with $k \in \N$.
By \cref{lem:chebyshev:generalised}, for any $\vz \in [-\lambda \sqrt{n},\lambda \sqrt{n}]^d$, the probability than no subset-sum is sufficiently close to $\vz$ is at most $\left(\frac{2}{3}\right)^k$.
Leveraging the fact that it is possible to cover $[-\lambda \sqrt{n}, \lambda \sqrt{n}]^d$ by $\left(\frac{\lambda \sqrt{n}}{\varepsilon}\right)^d$ hypercubes of radius $\varepsilon$, we can ensure that the probability of failing to $2\varepsilon$-approximate any $\vz$ is, by the union bound, at most
\begin{align}
    \left(\frac{\lambda \sqrt{n}}{\varepsilon}\right)^d\cdot \left(\frac{2}{3}\right)^k & = 2 ^{- k \log \frac{3}{2} + d \left( \log \frac{1}{\varepsilon} + \frac{1}{2}\log n + \log \lambda \right)} \nonumber \\
    & =\exp \left[{ - \ln 2\cdot  \frac{n - \frac{144 d^2}{\alpha^2 \log \frac{3}{2}} \left(\log \frac{1}{\varepsilon} + \frac{1}{2} \log n + \log \lambda \right) \left(\log \frac{1}{\varepsilon} + \log d + \log \frac{1}{\alpha} \right)}{\frac{144 d}{\alpha^2 \log \frac{3}{2}}\left(\log \frac{1}{\varepsilon} + \log d + \log \frac{1}{\alpha}\right)} }\right] \nonumber \\
    & \le \exp \left[{ - \ln 2\cdot  \frac{n - \frac{144 d^2}{\alpha^2 \log \frac{3}{2}} \left(\log \frac{1}{\varepsilon} + \frac{1}{2} \log n  \right) \left(\log \frac{1}{\varepsilon} + \log d + \log \frac{1}{\alpha} \right)}{\frac{144 d}{\alpha^2 \log \frac{3}{2}}\left(\log \frac{1}{\varepsilon} + \log d + \log \frac{1}{\alpha}\right)} }\right]\label{eq:generalisation1:toprove1}
\end{align}
since $\lambda < 1$. Consider $\frac{n}{2} - \frac{144d^2}{2\alpha^2 \log \frac{3}{2}} \log n \left(\log \frac{1}{\varepsilon} + \log d + \log \frac{1}{\alpha}\right) $. Let $ k = k' \left(\log\frac{1}{\varepsilon} + \log d + \log\frac{1}{\alpha}\right)$, which means that
$n = \frac{144k' d}{\alpha^2} \left(\log\frac{1}{\varepsilon} + \log d + \log\frac{1}{\alpha}\right)^2$.
Then
\begin{align*}
    &\frac{n}{2} - \frac{144d^2}{2\alpha^2 \log \frac{3}{2}} \log n \left(\log \frac{1}{\varepsilon} + \log d + \log \frac{1}{\alpha}\right)
    \\&\qquad= \frac{144d}{2\alpha^2 } \left(\log \frac{1}{\varepsilon} + \log d + \log \frac{1}{\alpha}\right) \left[k'\left(\log \frac{1}{\varepsilon} + \log d + \log \frac{1}{\alpha}\right)\right.
    \\&\qquad\qquad\left. - \frac{d}{\log \frac{3}{2}}\left(\log \frac{144}{\log \frac{3}{2}} + \log k' + \log d + 2 \log \frac{1}{\alpha} + 2 \log \left( \log \frac{1}{\varepsilon} + \log d + \log \frac{1}{\alpha}\right)\right)\right]
    \\&\qquad\ge \frac{144d}{2\alpha^2 } \left(\log \frac{1}{\varepsilon} + \log d + \log \frac{1}{\alpha}\right) \left[k'\left(\log \frac{1}{\varepsilon} + \log d + \log \frac{1}{\alpha}\right)\right.
    \\&\qquad\qquad\left. - 2d\left(8 + \log k' + \log d + 2 \log \frac{1}{\alpha} + 2 \log \left( \log \frac{1}{\varepsilon} + \log d + \log \frac{1}{\alpha}\right) \right)\right]
\end{align*}
If $k' = 17d$, we have that
\begin{align*}
    &k'\left(\log \frac{1}{\varepsilon} + \log d + \log \frac{1}{\alpha}\right) - 2d\left(8 + \log k' + \log d + 2 \log \frac{1}{\alpha} + 2 \log \left( \log \frac{1}{\varepsilon} + \log d + \log \frac{1}{\alpha}\right) \right)
    \\&\qquad\ge 4d\left(\log \frac{1}{\varepsilon} + \log d + \log \frac{1}{\alpha} - \log \left(\log \frac{1}{\varepsilon} + \log d + \log \frac{1}{\alpha}\right) \right) + 13d\log d + 13d \log \frac{1}{\alpha}
    \\&\qquad\qquad - 16d - 2d\log c - 3d\log d - 4d \log \frac{1}{\alpha}
    \\&\qquad= 10d\log d + 9d\log \frac{1}{\alpha} - 16d - 2d\log 17 \ge 0,
\end{align*}
as $\alpha \le \frac{1}{6}$. Thus, for $n \ge \frac{17\cdot 144 d^2}{\alpha^2} \left(\log \frac{1}{\varepsilon} + \log d + \log \frac{1}{\alpha}\right)^2$, we have that the expression in \cref{eq:generalisation1:toprove1} is at most
\begin{align*}
    \exp \left[{ - \ln 2\cdot  \frac{n - \frac{288 d^2}{\alpha^2 \log \frac{3}{2}} \log \frac{1}{\varepsilon} \left(\log \frac{1}{\varepsilon} + \log d + \log \frac{1}{\alpha} \right)}{\frac{288 d}{\alpha^2 \log \frac{3}{2}}\left(\log \frac{1}{\varepsilon} + \log d + \log \frac{1}{\alpha}\right)} }\right].
\end{align*}
We have the thesis by setting $\delta = \frac{288}{\log \frac{3}{2}}$ and $C = 17\cdot 144$.
\end{proof}

Our analysis, that relies on fixed subset sizes, easily extends \cref{thm:generalisation1} for non-centred and non-unitary normal random vectors.

\begin{corollary}\label{cor:generalisation2}
    Let $\sigma > 0$ and $\varepsilon \in (0, \sigma)$. Given $d, n \in \N$ let $\rvx_1, \dots, \rvx_n$ be independent normal $d$-dimensional random vectors with $\rvx_i \sim \NormalDist(\vv, \sigma^2\cdot\mI_d)$, for any vector $\vv \in \R^d$. Furthermore, let $\alpha \in \bigl(0, \frac{1}{6\sqrt{d}}\bigr)$.
    There exist two universal constants $C > \delta > 0$ such that, if
    $$
        n \ge C \frac{d^2}{\alpha^2} \left(\log\frac{\sigma}{\varepsilon} + \log d + \log\frac{1}{\alpha}\right)^2,
    $$
    then,
    with probability
    \begin{align*}
        1 - \exp \left[{- \ln 2 \cdot \left(\frac{n }{\delta \frac{d}{\alpha^2}\left(\log \frac{\sigma}{\varepsilon} + \log d + \log \frac{1}{\alpha}\right) } - d \log \frac{\sigma}{\varepsilon}\right)}\right],
    \end{align*}
    for all $\vz \in [-\sigma \lambda \sqrt{n},\sigma \lambda \sqrt{n}]^d + \alpha n \vv$, with $\lambda = \frac{1}{2}\sqrt{\frac{\alpha}{17d}}$, there exists a subset $S_{\vz} \subseteq [n]$ for which
    \begin{equation*}
        \bigl\lVert \vz - \sum_{i \in S_\vz} \rvx_i \bigr\rVert_\infty \le 2\varepsilon.
    \end{equation*}
    Moreover, this remains true even when restricted to subsets of size $\alpha n$.
\end{corollary}
\begin{proof}
    Simply apply \cref{thm:generalisation1} to the random vectors $ \frac{\rvx_i - \vv}{\sigma}$ with error $\frac{\varepsilon}{\sigma}$.
\end{proof}

Following the line of \cite{lueker-exponentially-1998}, we also observe that our results extend to a wider class of probability distributions.

\begin{definition}
    Consider any two random variables $\rx$ and $\ry$ having the same codomain, and let $\varphi_{\rx}(x), \varphi_{\ry}(x)$ be their probability density functions. We say that $\rx$ \emph{contains} $\ry$ with probability $p$ if a constant $p \in (0,1]$ exists such that $\varphi_{\rx}(x) = p \cdot \varphi_{\ry}(x) + (1-p) f(x)$ for any function $f(x)$.
\end{definition}

If $\rx$ contains $\ry$ with probability $p$, we can describe the behaviour of $\rx$ as follows: with probability $p$, draw $\ry$; with probability $1-p$, draw something else.
An adapted version of our result holds for random variables containing Gaussian distributions.

\begin{corollary}\label{cor:generalisation3}
    Let $\sigma > 0$, $\varepsilon \in (0, \sigma)$, and let $p \in (0,1]$ be a constant. Given $d, n \in \N$ let $\rvy_1, \dots, \rvy_n$ be independent $d$-dimensional random vectors containing $d$-dimensional normal random vectors $\rvx \sim \NormalDist(\vv, \sigma^2\cdot\mI_d) $ with probability $p$, where $\vv$ is any vector in $\R^d$ . Furthermore, let $\alpha \in \bigl(0, \frac{1}{6\sqrt{d}}\bigr)$.
    There exist two universal constants $C > \delta > 0$ such that, if
    $$
        n \ge 2C \frac{d^2}{p\alpha^2} \left(\log\frac{\sigma}{\varepsilon} + \log d + \log\frac{1}{\alpha}\right)^2,
    $$
    then,
    with probability
    \begin{align*}
        1 - 2\exp \left[{- \ln 2 \cdot \left(\frac{pn }{2\delta \frac{d}{\alpha^2}\left(\log \frac{\sigma}{\varepsilon} + \log d + \log \frac{1}{\alpha}\right) } - d \log \frac{\sigma}{\varepsilon}\right)}\right],
    \end{align*}
    for all $\vz \in \bigl[-\sigma \lambda \sqrt{\frac{pn}{2}},\sigma \lambda \sqrt{\frac{pn}{2}}\bigr]^d + \frac{\alpha p n }{2} \vv$, with $\lambda = \frac{1}{2}\sqrt{\frac{\alpha}{17d}}$, there exists a subset $S_{\vz} \subseteq [n]$ for which
    \begin{equation*}
        \bigl\lVert \vz - \sum_{i \in S_\vz} \rvx_i \bigr\rVert_\infty \le 2\varepsilon.
    \end{equation*}
    Moreover, this remains true even when restricted to subsets of size $ \frac{\alpha pn}{2}$.
\end{corollary}
\begin{proof}
    With a simple application of the Chernoff bound, we have that at least $ \frac{pn}{2} $ random vectors are normal random vectors with probability $1 - e^{-\frac{pn}{8}}$. Conditional on this event, we can apply \cref{cor:generalisation2} to the $\frac{pn}{2}$ normal random vectors. Since $\pr{A, B} \ge \pr{A \given B}\pr{B}$ for any two events $A,B$, and $ 2\delta \frac{d}{\alpha^2}\left(\log \frac{\sigma}{\varepsilon} + \log d + \log \frac{1}{\alpha}\right) \ge 8 $, the thesis holds with probability at least
    \begin{align*}
         &1 - \exp \left[{- \ln 2 \cdot \left(\frac{pn }{2\delta \frac{d}{\alpha^2}\left(\log \frac{\sigma}{\varepsilon} + \log d + \log \frac{1}{\alpha}\right) }
         - d \log \frac{\sigma}{\varepsilon}\right)}\right] - \exp\left[-\frac{pn}{8}\right]
         \\&\qquad\ge 1 -  2\exp \left[{- \ln 2 \cdot \left(\frac{pn }{2\delta \frac{d}{\alpha^2}\left(\log \frac{\sigma}{\varepsilon} + \log d + \log \frac{1}{\alpha}\right) }
         - d \log \frac{\sigma}{\varepsilon}\right)}\right].
    \end{align*}
\end{proof}

 \section{Discrete setting}
\label{apx:discrete}

We believe that it should not be hard to adapt our proof to several discrete distributions, in order to obtain results similar to those discussed in the Related Work section.
We also note that our Theorem \ref{thm:multidim-rss-ours} already implies an analogous discrete result.
Suppose that we quantise our random vectors by truncating them to the $\lfloor \log \frac 1\delta\rfloor$-th binary place, obtaining vectors $\hat \rvx_i $ such that $ \lVert \hat \rvx_i - \rvx_i \rVert_\infty < \delta$.
For any $\vz \in [-1,1]^d$, Theorem \ref{thm:multidim-rss-ours} guarantees that w.h.p. there is  a subset of indices $I\subseteq [n]$
such that $\lVert \vz - \sum_{i\in I} \rvx_i \rVert_\infty < \varepsilon $ and, hence, by the triangular inequality,
$\lVert \vz - \sum_{i\in I} \hat\rvx_i \rVert_\infty < n\delta + 2\varepsilon$.
As a special case ($\delta = 2\varepsilon$), we have the following:
\begin{corollary}[Discretization of Theorem \ref{thm:multidim-rss-ours}]
    Given $d \in \N$, $\varepsilon \in (0, 1)$, let $\hat \rvx_1, \dots, \hat \rvx_n$ be independent standard normal $d$-dimensional vectors truncated to the $\lfloor \log \frac 1\varepsilon\rfloor$-th binary place.
    There exists a universal constant $C > 0$ such that, if $n \ge C d^3\log \frac{1}{\varepsilon} \left(\log \frac{1}{\varepsilon} + \log d\right)$, then, with high probability, for all vectors $\hat \vz$ with entries in $\{k\varepsilon\}_{\lceil -\frac 1\varepsilon\rceil \leq k \leq \lfloor \frac 1\varepsilon \rfloor}$ there exists a subset $S_\vz \subseteq [n]$ for which
    \begin{equation*}
        \bigl\lVert \hat \vz - \sum_{i \in S_{\vz}} \hat \rvx_i \bigr\rVert_\infty \le 2\varepsilon (n+1).
    \end{equation*}
        Moreover, the approximation can be achieved with subsets of size $\frac{n}{6\sqrt{d}}$.
\end{corollary}
 \section{Connection with non-deterministic random walks}
\label{apx:nondetrndwalks}

Consider a discrete-time stochastic process whose state space is $\R^d$ which starts at the origin.
At the first step, the process ``branches'' in two processes, one of which keeps still, while the other moves by the vector $\rvx_{1}$. Recursively, given any time $i$ and any process, at the next time step the process branches in two other processes, one of which keeps still, while the other moves by the vector $\rvx_{i+1}$.
In this setting, when $\rvx_{i+1}$ are sampled from a standard multivariate normal distribution, our results imply that the resulting process is space filling: the process eventually gets arbitrarily close to each point in $\R^d$.
This should be contrasted with the fact that a Brownian motion is transient in dimension $d\geq 3$ \citep{morters2010brownian}.
The above process can also be interpreted as a multi-dimensional version of nondeterministic walks as introduced in \cite{PLW19} in the context of the analysis of encapsulations and decapsulations of network protocols, where the $i$-th $N$-step is~$\{\rvx_i, \vec 0\}$.
 
\end{document}